%% file: sample_paper.tex
\documentclass[twoside]{article}

\usepackage[preprint]{aistats2026}
\usepackage[round]{natbib}

\usepackage{url}
\usepackage{xcolor}
\definecolor{aistatsblue}{RGB}{0, 20, 115}

\usepackage[colorlinks=true,
            citecolor=aistatsblue,
            linkcolor=aistatsblue,
            urlcolor=aistatsblue]{hyperref}

\usepackage{algorithm}
\usepackage{algorithmic}
\usepackage{float}
\usepackage{amsfonts}
\usepackage{tikz}
\usepackage{amsmath}
\usepackage{amsthm}
\usepackage{thmtools, thm-restate}

\DeclareMathOperator*{\argmax}{arg\,max}
\DeclareMathOperator*{\argmin}{arg\,min}

\newcommand{\mc}{\mathcal}
\newcommand{\mb}{\mathbb}

\newcommand{\s}{\text{start}}
\newcommand{\e}{\text{end}}

\newcommand{\End}{\textit{end}}
\newcommand{\Start}{\textit{start}}
\newcommand{\tdl}{\textit{true deferral loss}}
\newcommand{\sdl}{\textit{surrogate deferral loss}}
\usetikzlibrary{arrows.meta}
\usetikzlibrary{shapes, arrows, positioning, calc}
\usetikzlibrary{positioning,fit,arrows.meta}
\usepackage{fontawesome5}
\usepackage{array, booktabs, siunitx}
\usepackage{pgfplots}
\usetikzlibrary{pgfplots.groupplots} 
\pgfplotsset{compat=1.17}

\definecolor{cbBlue}{RGB}{31,119,180}
\definecolor{cbOrange}{RGB}{255,127,14}
\definecolor{cbGreen}{RGB}{44,160,44}
\definecolor{cbRed}{RGB}{214,39,40}
\definecolor{cbMagenta}{RGB}{148,103,189}
\usepackage{subcaption}

\begin{document}

%

%
\runningtitle{Optimal Query Allocation in EQA with LLMs} 
\runningauthor{Montreuil, Yeo, Carlier, Ng, Ooi}   

\twocolumn[

\aistatstitle{Optimal Query Allocation in Extractive QA with LLMs: A Learning-to-Defer Framework with Theoretical Guarantees}

\aistatsauthor{
Yannis Montreuil$^{*,1,4,5}$ \And
Shu Heng Yeo$^{*,1}$ \And
Axel Carlier$^{2,4}$ \AND
Lai Xing Ng$^{3,4}$ \And
Wei Tsang Ooi$^{1,4}$
}

\aistatsaddress{
$^{1}$School of Computing, National University of Singapore, Singapore \\
$^{2}$Fédération ENAC ISAE-SUPAERO ONERA, Université de Toulouse, France \\
$^{3}$Institute for Infocomm Research (A*STAR), Singapore \\
$^{4}$IPAL, IRL 2955, Singapore \\
$^{5}$CNRS@CREATE LTD, 1 Create Way, Singapore
}

]

\begin{abstract}
Large Language Models (LLMs) excel at generative language tasks but remain unreliable for structured prediction, particularly in extractive question answering (EQA), where success depends on precise span selection. These challenges are amplified in resource-constrained environments, such as mobile or embedded systems, where deploying high-capacity models is often infeasible. We propose a Learning-to-Defer framework that routes EQA queries across a pool of models with varying capabilities and costs to balance accuracy and efficiency. Our approach is grounded in statistical decision theory: we define a differentiable surrogate loss whose minimizer provably converges to the Bayes-optimal allocation policy. Experiments on SQuADv1, SQuADv2, and TriviaQA show that our method consistently improves the accuracy-efficiency trade-off relative to static baselines and prior routing heuristics. Overall, our framework provides a principled and scalable solution for EQA in both high-performance and on-device deployment settings.
\end{abstract}

\input{Sections/Introduction}

\input{Sections/RelatedWorks}
\input{Sections/Preliminaries}
\input{Sections/Methods}
\input{Sections/Experiments}
\input{Sections/Conclusion}

\bibliographystyle{apalike}
\bibliography{aistats2026}

\newpage
\section*{Checklist}

\begin{enumerate}

  \item For all models and algorithms presented, check if you include:
  \begin{enumerate}
    \item A clear description of the mathematical setting, assumptions, algorithm, and/or model. [Yes]
    \item An analysis of the properties and complexity (time, space, sample size) of any algorithm. [Yes]
    \item (Optional) Anonymized source code, with specification of all dependencies, including external libraries. [Yes]
  \end{enumerate}

  \item For any theoretical claim, check if you include:
  \begin{enumerate}
    \item Statements of the full set of assumptions of all theoretical results. [Yes]
    \item Complete proofs of all theoretical results. [Yes]
    \item Clear explanations of any assumptions. [Yes]     
  \end{enumerate}

  \item For all figures and tables that present empirical results, check if you include:
  \begin{enumerate}
    \item The code, data, and instructions needed to reproduce the main experimental results (either in the supplemental material or as a URL). [Yes]
    \item All the training details (e.g., data splits, hyperparameters, how they were chosen). [Yes]
    \item A clear definition of the specific measure or statistics and error bars (e.g., with respect to the random seed after running experiments multiple times). [Yes]
    \item A description of the computing infrastructure used. (e.g., type of GPUs, internal cluster, or cloud provider). [Yes]
  \end{enumerate}

  \item If you are using existing assets (e.g., code, data, models) or curating/releasing new assets, check if you include:
  \begin{enumerate}
    \item Citations of the creator If your work uses existing assets. [Yes] 
    \item The license information of the assets, if applicable. [Yes]
    \item New assets either in the supplemental material or as a URL, if applicable. [Yes]
    \item Information about consent from data providers/curators. [Not Applicable]
    \item Discussion of sensible content if applicable, e.g., personally identifiable information or offensive content. [Not Applicable]
  \end{enumerate}

  \item If you used crowdsourcing or conducted research with human subjects, check if you include:
  \begin{enumerate}
    \item The full text of instructions given to participants and screenshots. [Not Applicable]
    \item Descriptions of potential participant risks, with links to Institutional Review Board (IRB) approvals if applicable. [Not Applicable]
    \item The estimated hourly wage paid to participants and the total amount spent on participant compensation. [Not Applicable]
  \end{enumerate}

\end{enumerate}

\clearpage
\appendix

\input{Sections/Appendix}

\thispagestyle{empty}

\end{document}

%% file: Sections/Introduction.tex
\section{INTRODUCTION}

Large Language Models (LLMs) have demonstrated strong performance on a wide range of natural language processing tasks, including translation, summarization, and question answering~\citep{touvron2023llamaopenefficientfoundation, jiang2023mistral7b, openai2024gpt4technicalreport}. Their broad generalization ability, acquired through large-scale pretraining, enables fluent, context-aware responses across diverse inputs. However, these generative strengths do not necessarily transfer to high-precision structured prediction tasks. A notable example is \emph{extractive question answering} (EQA), in which the model must identify an exact span from a given passage~\citep{chen2017reading, Alqifari, lan2020albertlitebertselfsupervised}. In such settings, LLMs often produce plausible but unsupported answers, undermining reliability~\citep{sadat2023delucionqadetectinghallucinationsdomainspecific}.

Recent advances in compression and distillation have made it feasible to deploy lightweight LLM variants in resource-constrained environments, such as mobile or embedded systems, where memory, latency, and compute are limited~\citep{sun2020mobilebertcompacttaskagnosticbert, merenda2020edge, lin2024awq, egashira2024exploiting}. However, such models still struggle with fine-grained reasoning tasks like EQA, where faithful span selection is essential. One might instead deploy a specialized EQA model, but such models are inflexible for general queries, limiting their practical utility. This creates a dilemma: lightweight LLMs are versatile but error-prone on structured tasks, whereas EQA models are precise but narrow in scope. Deploying both on the same device is often infeasible under resource constraints. This trade-off motivates an adaptive hybrid approach that leverages the strengths of both model types without requiring them to be co-located.

To address this limitation, we propose a deferral-based strategy that adaptively routes queries between a lightweight on-device LLM and one or more off-device specialized EQA models. The lightweight model handles simple or low-risk inputs locally, while complex or uncertain cases are deferred to more accurate models. This balances the generality and efficiency of small LLMs with the precision of expert QA systems without requiring all models to reside on the same device~\citep{devlin2018bert, liu2019roberta, lan2020albertlitebertselfsupervised}. We formalize this approach within a \emph{Learning-to-Defer} (L2D) framework~\citep{madras2018predict, mozannar2021consistent, Verma2022LearningTD, mao2023twostage, mao2024realizablehconsistentbayesconsistentloss, montreuil2024twostagelearningtodefermultitasklearning, montreuil2025adversarial, montreuil2026why}, where a learned policy assigns each input to the model offering the best accuracy-cost trade-off. Unlike prior heuristic or confidence-based routing, our policy minimizes a differentiable surrogate loss that provably converges to the Bayes-optimal allocation under mild conditions.

%% file: Sections/RelatedWorks.tex
\section{RELATED WORK}

\paragraph{Model Cascades.}
Model cascades~\citep{990517, jitkrittum2024doesconfidencebasedcascadedeferral, JMLR:v15:saberian14a} process a query through a sequence of models, forwarding it to the next stage only if a confidence-based criterion fails to meet a predefined threshold. These thresholds aim to balance predictive performance and computational cost. Although recent work has adapted cascades to LLMs~\citep{kolawole2024agreementbasedcascadingefficientinference, yue2023large}, such methods are not tailored to the EQA setting. Moreover, cascade-based designs often struggle to accommodate heterogeneous models—e.g., mixing span-predicting EQA models with free-form generative LLMs—due to incompatible output formats~\citep{varshney2022modelcascadingjointlyimproving}. As models are added to the cascade, inference latency increases and optimal predictions may be delayed. \emph{Agreement-Based Cascading}~\citep{narasimhan2024fastercascadesspeculativedecoding} uses ensemble agreement at each stage to decide whether to escalate the query. While this improves robustness, it still suffers from the limitations of sequential inference.

\paragraph{Query Routing.}
Query routing~\citep{ding2024hybridllmcostefficientqualityaware, ong2024routellmlearningroutellms, osti_10447699, 10.14778/3574245.3574273, stojkovic2025dynamollm, chen2024routerdc} aims to improve efficiency by learning to dispatch each query among models, trading off pools or instances of fast, low-capacity models against a slower, higher-accuracy alternative~\citep{chen2025harnessing}. Routing decisions are guided by estimates of input difficulty or task-specific quality requirements, and are particularly relevant in edge and resource-constrained deployments where latency and energy consumption are critical~\citep{qu2025mobile}. Recent work extends this paradigm beyond binary routing to allow selection among multiple candidate models~\citep{lu-etal-2024-routing, Ding2025BESTRouteAL}, which better reflects real-world deployment choices. However, existing approaches either lack Bayes-consistent guarantees, do not address structured span prediction settings, or neglect practical deployment costs such as latency and token expenditure.

\paragraph{Structured Output Abstention.} Structured Output Abstention~\citep{pmlr-v80-garcia18a} allows a model to withhold predictions on components of a structured output while incurring a pre-specified abstention cost. While seemingly related to query allocation, it addresses a different decision problem. Specifically, abstention focuses on \emph{when} a model should refrain from predicting. In contrast, query allocation determines which model or expert should produce the prediction, explicitly routing the input to another decision maker under a cost-quality trade-off.

\paragraph{Learning-to-Defer.}
Learning-to-Defer (L2D) frames query allocation as a principled classification-with-deferral problem, where the learner can route inputs either to a model or to one of several experts under a cost-quality trade-off~\citep{madras2018predict, mozannar2021consistent, Verma2022LearningTD}. Subsequent work has pursued refinements of surrogate design~\citep{charusaie2022sample, mao2024principledapproacheslearningdefer, montreuil2026beyond, montreuil2026learningtodeferexpertconditionedadvice}; theoretical guarantees such as $\mathcal{H}$-consistency and realizability~\citep{Mozannar2023WhoSP, mao2024realizablehconsistentbayesconsistentloss, mao2025mastering, mao2025thesis}; top-$k$ deferral and online or non-stationary regimes~\citep{montreuil2026why, montreuil2026online, montreuil2026learningdefernonstationarytime}; adversarial robustness~\citep{montreuil2025adversarial, montreuil2026adversarial}; and budgeted or imbalanced deferral~\citep{desalvo2025budgeted, cortes2026optimized}. Two-stage formulations decouple the predictor from the allocation policy~\citep{mao2023twostage, montreuil2024twostagelearningtodefermultitasklearning}. The consistency of our deferral surrogate ultimately inherits from $\mathcal{H}$-consistency bound theory~\citep{Awasthi_Mao_Mohri_Zhong_2022_multi, mao2024h}, whose guarantees have been instantiated for abstention and rejection~\citep{theoretically, Mao_Mohri_Zhong_2023, mohri2024learningreject}, for ranking and cardinality-aware set prediction---directly relevant to selecting answer spans in EQA~\citep{cortes2024cardinalityaware, mao2023pairwisemisranking, mao2023rankingabstention}, for structured outputs~\citep{mao2023structuredprediction}, and for adversarially robust surrogates~\citep{awasthi2021calibrationconsistencyadversarialsurrogate, Grounded}, with further algorithmic developments in generalized-metric optimization and robust generative modeling~\citep{mohri2026generalized, mohri2026principled, cortes2026theoretical}. Our work specializes this framework to extractive QA with LLMs, treating heterogeneous LLMs as experts in a one-stage allocation scheme.

\paragraph{Contributions.}
We advance EQA under resource constraints by unifying statistical decision theory with dynamic model selection. We propose a hybrid setting that combines the multi-expert flexibility of cascades with the direct-dispatch paradigm of routing. This generalization enables richer trade-offs by allowing queries to be dispatched directly across several experts with potentially different architectures.

(i) We introduce a new principled framework for cost-sensitive model selection in EQA, deriving a consistent, end-to-end trainable loss tailored to deferral-based architectures.
(ii) We establish formal guarantees showing that our learned deferral policy provably converges to the Bayes-optimal allocation, offering theoretical insight into the limits of adaptive routing.
(iii) Through extensive experiments on SQuADv1, SQuADv2, and TriviaQA, we demonstrate that our method consistently outperforms existing routing strategies and heuristic baselines.

%% file: Sections/Preliminaries.tex
\section{PRELIMINARIES} \label{preliminary}
\paragraph{Extractive QA.}
We consider EQA, where the answer \(a\) must be returned as a contiguous span in a context \(c\) given a question \(q\). Let the random pair \((X,Y)\sim\mathcal D\) denote a draw from an unknown data-generating distribution.
Throughout, upper-case symbols (e.g.\ \(X\)) denote random variables,
while lower-case symbols (e.g.\ \(x\)) denote their fixed realizations. A specific instance is \(x=(q,c)\in\mathcal X\).
Its label is the span
\(y=(y^{\text{start}},y^{\text{end}})\in\mathcal Y\) with
\(0\le y^{\text{start}}\le y^{\text{end}}<|c|\); hence
\(\mathcal Y\) factorises as
\(\mathcal Y^{\text{start}}\!\times\!\mathcal Y^{\text{end}}\).
Following prior work
\citep{devlin2018bert, liu2019roberta, lan2020albertlitebertselfsupervised},
we assume the start and end indices are conditionally independent given
the input.
All samples are i.i.d.\ according to \(\mathcal D\) \citep{Mohri}.
For later use, let \(\mathcal D^{i}\) be the marginal of \(\mathcal D\) on
\( \mathcal X\times\mathcal Y^{i}\); i.e., \((X,Y^i)\sim\mathcal D^{i}\).

The EQA model is defined as a parametric function \( g \in \mathcal{G} \), composed of a feature extractor \( w \in \mathcal{W} \) and a span predictor \( h \in \mathcal{H} \). The extractor maps inputs to latent representations via \( w: \mathcal{X} \to \mathcal{T} \), with \( t = w(x) \in \mathcal{T} \). These representations are scored by the classifier \( h = (h^{\text{start}}, h^{\text{end}}) \), where each head \( h^i: \mathcal{T} \times \mathcal{Y}^i \to \mathbb{R} \) defines a position-wise scoring function. Predictions are then made according to the maximization rule \( g^i(x) = \arg\max_{y \in \mathcal{Y}^i} h^i(w(x), y) \), where \( w \in \mathcal{W} \) and \( h^i \in \mathcal{H} \) denote a shared representation map and task-specific scoring function, respectively. The overall model is defined by composition as \( g = h \circ w \), inducing the function class \( \mathcal{G} = \{ g \mid g(x) = h \circ w(x),\; w \in \mathcal{W},\; h \in \mathcal{H} \} \).

Model training typically minimizes the \emph{joint 0--1 loss}, which counts the number of incorrect predictions across both span endpoints. This loss is defined as \( \ell_{01}^{\text{joint}}: \mathcal{Y} \times \mathcal{Y} \to \{0,1,2\} \), where
\begin{equation*}
\ell_{01}^{\text{joint}}(g(x), y) = \mathbf{1}\left[g^{\text{start}}(x) \neq y^{\text{start}}\right] + \mathbf{1}\left[g^{\text{end}}(x) \neq y^{\text{end}}\right].
\end{equation*}
Although non-differentiable, this loss provides an intuitive and interpretable measure of token-matching performance in EQA.

\paragraph{Bayes and $\mc{G}^i$-consistency.}
Let \( i \in \{\text{start}, \text{end}\} \). The learning objective is to find a predictor \( g^i \in \mathcal{G}^i \) that minimizes the expected 0--1 error,
\[
\mathcal{E}_{\ell_{01}}(g^{i})
  \;=\;
  \mathbb{E}_{(X,Y^i) \sim \mathcal D^i}
  \bigl[
     \ell_{01}\bigl(g^{i}(X),\,Y^{i}\bigr)
  \bigr].
\]
\noindent The Bayes-optimal error is defined as
\[
\mathcal{E}^B_{\ell_{01}}(\mathcal{G}^i) = \inf_{g^i \in \mathcal{G}^i} \mathcal{E}_{\ell_{01}}(g^i).
\]
However, direct minimization is intractable due to the discontinuity and non-convexity of the multiclass 0--1 loss \( \ell_{01} \)~\citep{Statistical, Steinwart2007HowTC, Awasthi_Mao_Mohri_Zhong_2022_multi, mao2024hconsistencyregression, mao2024multilabel, mao2024universalgrowth, mao2025enhanced, zhong2025thesis, mohri2026beyond, cortes2025improvedbalanced, cortes2025balancingscales, mao2025principledbinary, mohri2026linear, mohri2026mind}.

To overcome this, we adopt a family of convex surrogate losses \( \Phi_{01}^\nu: \mathcal{G}^i \times \mc{X} \times \mathcal{Y}^i \to \mathbb{R}_+ \), parameterized by \( \nu \geq 0 \), which upper bound \( \ell_{01} \). This family subsumes common losses such as log-softmax (for \( \nu = 1 \))~\citep{Mohri} and the mean absolute error (for \( \nu = 2 \))~\citep{Ghosh2017RobustLF}, and is defined as:
\begin{equation}\label{eq:comp-sum}
\Phi_{01}^\nu(g^i, x, y^i) =
\begin{cases}
\displaystyle \frac{1}{1 - \nu} \left( \Psi(g^i, x, y^i)^{1 - \nu} - 1 \right) & \text{if } \nu \neq 1, \\[8pt]
\displaystyle \log \Psi(g^i, x, y^i) & \text{if } \nu = 1,
\end{cases}
\end{equation}
where \( \Psi(g^i,x, y^i) = \sum_{y' \in \mathcal{Y}^i} \exp\left( g^i(x, y') - g^i(x, y^i) \right) \). The expected surrogate risk is given by
$
\mathcal{E}_{\Phi_{01}^\nu}(g^i) = \mathbb{E}_{(X,Y^i) \sim \mc{D}^i}[\Phi_{01}^\nu(g^i, X,Y^i)],
$ with corresponding infimum
$
\mathcal{E}^*_{\Phi_{01}^\nu}(\mathcal{G}^i) = \inf_{g^i \in \mathcal{G}^i} \mathcal{E}_{\Phi_{01}^\nu}(g^i).
$

A central property of surrogate losses is \emph{Bayes consistency}, which ensures that minimizing surrogate risk also minimizes true risk. Specifically, \( \Phi_{01}^\nu \) is Bayes-consistent with respect to \( \ell_{01} \) if, for any sequence \( \{g_k^i\}_{k \in \mathbb{N}} \subset \mathcal{G}^i \),
\begin{equation}\label{bayes-consi}
\begin{aligned}
\mathcal{E}_{\Phi_{01}^\nu}(g_k^i) - \mathcal{E}_{\Phi_{01}^\nu}^*(\mathcal{G}^i_{\text{all}}) \xrightarrow{k \to \infty} 0
\quad \Longrightarrow \\ \quad
\mathcal{E}_{\ell_{01}}(g_k^i) - \mathcal{E}_{\ell_{01}}^B(\mathcal{G}^i_{\text{all}}) \xrightarrow{k \to \infty} 0.
\end{aligned}
\end{equation}

While this implication holds when \( \mathcal{G}^i = \mathcal{G}^i_{\text{all}} \), it need not hold under hypothesis class restrictions such as \( \mathcal{G}^i_{\text{lin}} \) or \( \mathcal{G}^i_{\text{ReLU}} \)~\citep{pmlr-v28-long13, Awasthi_Mao_Mohri_Zhong_2022_multi}. To address this, \citet{Awasthi_Mao_Mohri_Zhong_2022_multi} introduce \emph{\( \mathcal{G}^i \)-consistency} bounds, which quantify surrogate-to-true error transfer via a monotonic function \( \Gamma: \mathbb{R}_+ \to \mathbb{R}_+ \):
\begin{equation}\label{mhbc}
\begin{aligned}
& \mathcal{E}_{\Phi_{01}^\nu}(g^i) - \mathcal{E}_{\Phi_{01}^\nu}^*(\mathcal{G}^i) + \mathcal{U}_{\Phi_{01}^\nu}(\mathcal{G}^i)
\geq \\
& \qquad \Gamma\left( \mathcal{E}_{\ell_{01}}(g^i) - \mathcal{E}_{\ell_{01}}^B(\mathcal{G}^i) + \mathcal{U}_{\ell_{01}}(\mathcal{G}^i) \right),
\end{aligned}
\end{equation}
where the minimizability gap $\mathcal U_{\ell_{01}}(\mathcal G^{i})
\;=\;
\mathcal E_{\ell_{01}}^{B}(\mathcal G^{i})
-\mathbb E_{X\sim\mathcal D^{i}_{X}}
\Bigl[
   \inf_{g^{i}\in\mathcal G^{i}}
      \mathbb E_{Y^{i}\sim\mathcal D^{i}(\,\cdot\mid X)}
      \bigl[\ell_{01}\bigl(g^{i}(X),Y^{i}\bigr)\bigr]
\Bigr]$ captures the irreducible bias induced by function class \( \mathcal{G}^i \). Notably, this gap vanishes when \( \mathcal{G}^i = \mathcal{G}^i_{\text{all}} \)~\citep{Steinwart2007HowTC, Awasthi_Mao_Mohri_Zhong_2022_multi}, in which case inequality~\eqref{mhbc} recovers the standard Bayes consistency property in~\eqref{bayes-consi}.

%% file: Sections/Methods.tex
\section{OPTIMAL ALLOCATION FOR EQA SYSTEMS} \label{method}

\begin{figure*}[t]
    \centering
    \begin{tikzpicture}[
        block/.style={
            rectangle, draw, rounded corners,
            align=center,
            minimum height=0.5cm,
            minimum width=2cm,
            font=\small
        },
        arrow/.style={-{Stealth[scale=1.2]}, thick}
    ]

    \node[block, fill=yellow!20] (q) {Question\\and Context};

    \node[block, fill=orange!20, right=1.3cm of q, yshift=1cm] (pr1) {Start\\Rejector};
    \node[block, fill=orange!20, right=1.3cm of q, yshift=-1cm] (pr2) {End\\Rejector};

    \node[draw, dashed, inner sep=10pt, rounded corners, fit=(pr1)(pr2)] (rf) {};
    \node[above=5pt of rf.north] {Rejector Framework};

    \draw[arrow] (q) -- (rf);

    \node[block, fill=purple!20, right=1cm of rf] (sel) {Optimal Decision\\$\hat{\pi}(x)$};
    \node[above=5pt of sel.north] {Cor.~\ref{cor:learn}};
    \draw[arrow] (rf) -- (sel);

    \node[block, fill=blue!20, right=1cm of sel, yshift=1.5cm] (m) {LLM};
    \node[block, fill=red!20, right=1cm of sel] (e1) {Expert 1};
    \node[block, fill=green!20, right=1cm of sel, yshift=-1.5cm] (e2) {Expert $J$};

    \path (e1) -- (e2) node[midway] {\vdots};

    \node[draw, dashed, inner sep=10pt, rounded corners, fit=(m)(e2)] (b1) {};
    \node[above=5pt of b1.north] {Available Agents};

    \draw[arrow] (sel) -- (b1);

    \node[block, fill=teal!20, right=1cm of e1] (pa) {Predicted Answer};
    \draw[arrow] (b1) -- (pa);

    \end{tikzpicture}
    \captionsetup{skip=2\baselineskip}
    \caption{
    Overview of our approach. Rejectors estimate start and end uncertainty, inducing an optimal allocation policy $\hat{\pi}(x)$ (Corollary~\ref{cor:learn}) that routes the query to an appropriate agent for answer prediction (Algorithm~\ref{algo_inference1}).
    }
\end{figure*}

In this section, we formalize the problem of allocating queries \( x \in \mathcal{X} \) across a set of agents comprising a primary model \( g \) and a collection of \( J \) expert models. Our objective is to learn an allocation policy that assigns each query to the agent most likely to produce a correct prediction while controlling overall computational cost.

Importantly, we show that the proposed formulation admits a Bayes-optimal solution: under mild assumptions, there exists an allocation strategy that asymptotically minimizes expected error. This result provides the theoretical foundation for our learning-to-defer framework and guarantees that the learned deferral policy approaches optimal performance in the limit.

\subsection{Formulating the Allocation Problem}

\paragraph{Setting.}
We consider a query allocation setup involving a primary model \( g \in \mathcal{G} \) and a collection of \( J \) pre-trained expert models, collectively referred to as \emph{agents}. The set of available agents is indexed by \( \mathcal{A} = \{0\} \cup [J] \), where agent \( 0 \) corresponds to the main model and \( [J] = \{1, \dots, J\} \) indexes the experts, yielding \( J + 1 \) agents in total. All agents are assumed to be fixed and trained offline; the objective is to learn a deferral policy that dynamically allocates each query \( x \in \mathcal{X} \) to one of the agents at inference time~\citep{mao2023twostage, mao2024regressionmultiexpertdeferral, mao2024realizablehconsistentbayesconsistentloss, mao2025mastering, montreuil2024twostagelearningtodefermultitasklearning, montreuil2025adversarial, montreuil2026why}. Each expert \( \mathrm{M}_j \), when queried on an input \( x \), produces a span prediction in the form of start and end token indices, denoted \( m_j^{\text{start}}(x) \in \mathcal{Y}^{\text{start}} \) and \( m_j^{\text{end}}(x) \in \mathcal{Y}^{\text{end}} \), respectively. These agents may correspond to human annotators, pretrained neural models, or other predictive systems. We denote the full set of expert predictions as \( m(x) = (m_1(x), \dots, m_J(x)) \in \mathcal{M} \), where each \( m_j(x) = (m_j^{\text{start}}(x), m_j^{\text{end}}(x)) \) represents the answer span returned by expert \( j \).

\paragraph{True Deferral Loss.}
To enable cost-sensitive allocation of queries among multiple \emph{agents}, we define a \emph{rejector} \( r \in \mathcal{R} \), which maps each input \( x \in \mathcal{X} \) to an agent index in \( \mathcal{A} \). The rejector is decomposed into two components—\( r^{\text{start}} \in \mathcal{R}^{\text{start}} \) and \( r^{\text{end}} \in \mathcal{R}^{\text{end}} \)—corresponding to independent deferral decisions for the start and end span predictions. Each \( r^i \in \mathcal{R}^i \), for \( i \in \{\text{start}, \text{end}\} \), is a scoring function \( r^i: \mathcal{X} \times \mathcal{A} \to \mb{R} \) that selects the agent with maximal score:
\[
r^i(x) = \arg\max_{j \in \mathcal{A}} r^i(x, j).
\]
To learn such rejectors, we adopt the \emph{True Deferral Loss (TDL)} from~\citet{mao2023twostage}, adapted here to structured prediction in EQA.

\begin{restatable}[True Deferral Loss]{definition}{deferral}
\label{def:true_deferral_loss}
Given an input \( x \in \mathcal{X} \) and a rejector \( r \in \mathcal{R} \), the true deferral loss is defined as
\[
\ell_{\text{def}}(r(x), y) =\mspace{-23mu}  \sum_{i \in \{\text{start}, \text{end}\}} \sum_{j = 0}^J c_j(x,y^i)  \mathbf{1}\{r^i(x) = j\},
\]
\end{restatable}

\noindent where \( c_j \) denotes the cost of assigning input \( x \) to agent \( j \). For the main model \( g \), this cost is defined as
\[
c_0\big(x,y^i\big) = \mathbf{1}\big\{g^i(x) \neq y^i\big\},
\]
which penalizes incorrect predictions. For expert \( j > 0 \), the cost incorporates both prediction error and invocation penalty:
\[
c_j\big(x,y^i\big) = \alpha_j  \mathbf{1}\{m_j^i(x) \neq y^i\} + \beta_j,
\]
where \( \alpha_j \geq 0 \) scales the error penalty and \( \beta_j \geq 0 \) models the cost of consultation. Notably, setting \( \alpha_j = 0 \) reduces expert \( j \) to an oracle: always correct but not free to query~\citep{Chow_1970, cortes}.

The deferral decision thus hinges on minimizing expected deferral cost while maintaining high accuracy. When \( r^i(x) = 0 \), the query is assigned to the main model; otherwise, if \( r^i(x) = j > 0 \), it is deferred to expert \( j \), whose prediction \( m_j^i(x) \) incurs a penalty based on both error and query cost. A principled deferral strategy must therefore balance predictive reliability with the expense of expert consultation.
\subsection{Optimality of the Allocation}
An ideal deferral strategy allocates each query \( x \in \mathcal{X} \) to the agent most likely to predict correctly, thereby maximizing reliability and decision confidence~\citep{madras2018predict}. To formalize this intuition, we analyze the Bayes-optimal risk under the \tdl{} and characterize the \emph{Bayes-rejector}—the rejection function that minimizes expected deferral cost.

Let $\mathcal D^{i}(\,\cdot\mid X=x)$
denote the conditional distribution of the label
\(Y^{i}\) given an input \(x\), where
\((X,Y^{i})\sim\mathcal D^{i}\).  For any \(x\in\mathcal X\) we define the per-task, per-agent
\emph{conditional risks}
\begin{align}
\eta_{j}^{i}(x)=
  \mathbb E_{Y^i \sim \mc{D}^i(\cdot|X=x)}\!\bigl[c_{j}(x,Y^{i})\bigr]
\end{align}
which leads to $\eta_{0}^{i}(x) = 
  \Pr_{Y^{i}\sim\mathcal D^{i}(\cdot\mid X=x)}
     \bigl[g^{i}(x)\neq Y^{i}\bigr]$
and for expert $\eta_{j}^{i}(x) =
  \alpha_{j}\,
  \Pr_{Y^{i}\sim\mathcal D^{i}(\cdot\mid X=x)}
     \bigl[m_{j}^{i}(x)\neq Y^{i}\bigr]
  + \beta_{j}$, where \(\alpha_{j}\ge 0\) scales the expert’s prediction error and
\(\beta_{j}\ge 0\) is its fixed query cost.

\begin{restatable}[Bayes-Rejector]{lemma}{rejector}
\label{lemma:rejector}
Given an input \( x \in \mathcal{X} \) and any distribution \( \mathcal{D} \), the Bayes-optimal rejector that minimizes the conditional \tdl{} is
\[
r^{B,i}(x) =
\begin{cases}
0, & \text{if } \inf_{g^i \in \mathcal{G}^i} \eta_0^i(x) \leq \min_{j \in [J]} \eta_j^i(x), \\
j^\ast, & \text{otherwise},
\end{cases}
\]
where $j^\ast = \arg\min_{j \in [J]} \eta_j^i(x)$.
\end{restatable}

\noindent A proof is provided in the appendix. Lemma~\ref{lemma:rejector} formalizes the decision rule induced by the true deferral loss: defer only when an expert exhibits strictly lower expected risk than the main model. This ensures that each query is routed to the lowest-risk agent under the conditional distribution induced by \( \mathcal{D} \), yielding asymptotically optimal deferral performance.

\paragraph{Computational Challenge.}
While Lemma~\ref{lemma:rejector} prescribes the optimal deferral strategy, learning the Bayes-rejector is computationally intractable in practice. The problem is known to be NP-hard~\citep{Zhang, Steinwart2007HowTC, bartlett1, Mohri}, due to the discontinuity and non-convexity of the \tdl{}, which complicates optimization. This difficulty is characteristic of many structured prediction tasks where exact minimization of non-differentiable loss functions is infeasible~\citep{cortes, mao2023crossentropylossfunctionstheoretical}. In the next subsection, we propose a surrogate formulation that approximates the Bayes-rejector while preserving its theoretical guarantees.

\subsection{Accurate Approximation of the True Deferral Loss}  
To approximate the \tdl{} while preserving the optimality of the decision rule in Lemma~\ref{lemma:rejector}, we leverage tools from consistency theory as formalized in Section~\ref{preliminary}. A standard approach in statistical learning is to introduce a \emph{surrogate loss}, that is, a differentiable proxy for a target loss. In our setting, the goal is to construct a surrogate for the \tdl{} that is both Bayes-consistent and \((\mathcal{G}, \mathcal{R})\)-consistent, ensuring that minimizing the surrogate yields a rejector \( r^{\ast, i} \) that converges to the Bayes-optimal rejector from Lemma~\ref{lemma:rejector}. In particular, we aim to guarantee that, as the surrogate loss is minimized, the learned deferral strategy asymptotically approaches the optimal allocation rule.

\paragraph{Formulating the Surrogate Deferral Loss.}
To construct a tractable alternative to the non-differentiable true deferral loss (TDL), we adopt the cross-entropy multiclass surrogate family \( \Phi_{01}^\nu: \mathcal{R}^i \times \mathcal{X} \times \mathcal{A} \to \mathbb{R}_+ \), which upper bounds the multiclass 0--1 loss and enjoys favorable optimization properties. Following~\citet{mao2023twostage}, we adapt this surrogate to our structured EQA setting to define the \emph{Surrogate Deferral Loss (SDL)}.

\begin{restatable}[Surrogate Deferral Loss]{lemma}{surrogate}
\label{lemma:surrogate}
Given an input \( x \in \mathcal{X} \) and a labeled instance \((x, y) \), the surrogate deferral loss is defined as
\[
\Phi_{\mathrm{def}}^\nu(r, x,y) = \sum_{i \in \{\text{start}, \text{end}\}} \sum_{j = 0}^{J} \tau_j(x,y^i)  \Phi_{01}^\nu\big(r^i, x, j\big),
\]
where \( \tau_j(x,y^i) = \sum_{q \neq j} c_q(x,y^i) \) quantifies the nonnegative relative cost gap between agent \( j \) and the competing agents.
\end{restatable}
We give the proof of Lemma~\ref{lemma:surrogate} in the appendix. The term \( \tau_j(x,y^i) \) reflects a soft preference for agent \( j \) by measuring how much worse the alternatives are in terms of cost. Intuitively, minimizing the SDL encourages the rejector \( r^i \) to assign queries to agents with lower relative cost. The surrogate loss \( \Phi_{01}^\nu \) is typically instantiated as a log-softmax and serves as a smooth approximation to the discontinuous 0--1 decision boundary.

\vspace{\baselineskip}
\begin{algorithm}[ht]
   \caption{Training the rejector $r$}\label{algo_training1}
\begin{algorithmic}
   \STATE {\bfseries Input:} Dataset $\{(x_k, y_k^\s, y_k^\e)\}_{k=1}^K$, multi-task model $g\in\mc{G}$, experts $m\in\mc{M}$, rejectors $r=(r^\s, r^\e)$, number of epochs $\text{EPOCH}$, batch size BATCH, learning rate $\lambda$, surrogate parameter $\nu$. 
   \STATE {\bfseries Initialization:} Initialize  parameters $\theta=(\theta^\s, \theta^\e)$.
   \FOR{$i=1$ to $\text{EPOCH}$}
       \STATE Shuffle dataset $\{(x_k, y_k^\s, y_k^\e)\}_{k=1}^K$.
       \FOR{each $\mathcal{B} \subset \{(x_k, y_k^\s, y_k^\e)\}_{k=1}^K$ of size BATCH}
           \STATE Extract input-output pairs $z=(x, y^\s, y^\e) \in \mathcal{B}$.
           \STATE Query model $g(x)$ and experts $m(x)$. 
           \STATE Evaluate costs $c_j(x, y^\s)$ and $c_j(x, y^\e)$
           \STATE Compute the regularized empirical risk minimization:
           \STATE \hspace{1em} $\widehat{\mc{E}}_{\Phi_{\text{def}}}(r;\theta) = \frac{1}{\text{BATCH}} \sum_{z \in \mathcal{B}} \Big[ \Phi_{\text{def}}^\nu(r,x,y) \Big]$.
           \STATE Update parameters $\theta$:
           \STATE \hspace{1em} $\theta \leftarrow \theta - \lambda \nabla_\theta \widehat{\mc{E}}_{\Phi_{\text{def}}}(r;\theta)$. 
       \ENDFOR
   \ENDFOR
   \STATE \textbf{Return:} trained rejector model $\hat{r}$.
\end{algorithmic}
\end{algorithm}
\vspace{\baselineskip}

This surrogate formulation preserves the agent-comparison structure of the TDL while introducing differentiability, enabling end-to-end training via stochastic gradient descent. As such, it supports efficient integration with standard deep learning frameworks~\citep{bartlett1}, allowing scalable learning of deferral strategies. We give the detailed training procedure in Algorithm~\ref{algo_training1}. 

In the next subsection, we analyze the theoretical guarantees of the SDL. Specifically, we establish that under suitable conditions, minimizing \( \Phi_{\mathrm{def}}^\nu \) yields a rejector \( r^i \) that approximates the Bayes-optimal deferral rule. This \emph{Bayes consistency} ensures that our surrogate not only facilitates optimization but also preserves statistical optimality in the asymptotic regime.

\subsection{Theoretical Guarantees of the Surrogate Deferral Loss}

\paragraph{Consistency Guarantees.} In the previous subsection, we introduced the \sdl{} as a differentiable surrogate to approximate the true deferral loss \tdl{}. We now establish that minimizing the surrogate excess risk leads to a reduction in the true excess deferral risk. That is,
\[
\mathcal{E}_{\Phi_{\text{def}}^\nu}(r) - \mathcal{E}_{\Phi_{\text{def}}^\nu}^\ast(\mathcal{R}) + \mathcal{U}_{\Phi_{\text{def}}^\nu}(\mathcal{R})
\]
acts as a valid upper bound for
\[
\mathcal{E}_{\ell_{\text{def}}}(g, r) - \mathcal{E}_{\ell_{\text{def}}}^B(\mathcal{G}, \mathcal{R}) + \mathcal{U}_{\ell_{\text{def}}}(\mathcal{G}, \mathcal{R}),
\]
thereby implying that any minimizer \( r^\ast \in \mathcal{R} \) of the surrogate loss closely approximates the Bayes-optimal rejector \( r^B \), as defined in Lemma~\ref{lemma:rejector}.

\begin{restatable}[$(\mathcal{R}, \mathcal{G})$-Consistency]{theorem}{theoconsistency}
\label{theo:consistency}
Let \( x \in \mathcal{X} \) and let \( \mathcal{D} \) denote any distribution over \( \mathcal{X} \times \mathcal{Y} \). Fix any \( \nu \geq 0 \), and suppose there exists a non-decreasing, concave function \( \Gamma^\nu : \mathbb{R}^+ \to \mathbb{R}^+ \) such that, for all \( r \in \mathcal{R} \),
\begin{equation*}
    \begin{aligned}
        & \mathcal{E}_{\Phi_{01}^\nu}(r) - \mathcal{E}^*_{\Phi_{01}^\nu}(\mathcal{R}) + \mathcal{U}_{\Phi_{01}^\nu}(\mathcal{R}) \geq \\
& \qquad\Gamma^\nu\left(\mathcal{E}_{\ell_{01}}(r) - \mathcal{E}^B_{\ell_{01}}(\mathcal{R}) + \mathcal{U}_{\ell_{01}}(\mathcal{R})\right)
    \end{aligned}   
\end{equation*}
Then, for any \( (g, r) \in \mathcal{G} \times \mathcal{R} \),
\begin{equation*}
    \begin{aligned}
& \mathcal{E}_{\ell_{\text{def}}}(g, r) - \mathcal{E}^B_{\ell_{\text{def}}}(\mathcal{G}, \mathcal{R}) + \mathcal{U}_{\ell_{\text{def}}}(\mathcal{G}, \mathcal{R})
\leq \\
& \qquad \overline{\Gamma}^\nu\left( \mathcal{E}_{\Phi^\nu_{\text{def}}}(r) - \mathcal{E}^*_{\Phi^\nu_{\text{def}}}(\mathcal{R}) + \mathcal{U}_{\Phi^\nu_{\text{def}}}(\mathcal{R}) \right) \\
&\qquad+ \sum_{i \in \{\s, \e\}} \left( \mathcal{E}_{c_0}(g^i) - \mathcal{E}^B_{c_0}(\mathcal{G}^i) + \mathcal{U}_{c_0}(\mathcal{G}^i) \right),
\end{aligned}
\end{equation*}
where the rescaled function is defined by
\begin{equation*}
    \begin{aligned}
        \overline{\Gamma}^\nu(u) = \left( \sum_{i \in \{\s, \e\}} \mspace{-30mu}\|\boldsymbol{\overline{\tau}^i}\|_1 \right)\Gamma^\nu\left( \frac{u}{\sum_{i \in \{\s, \e\}} \|\boldsymbol{\overline{\tau}^i}\|_1} \right),
    \end{aligned}
\end{equation*}
and where the expected cost weights are given by \( \boldsymbol{\overline{\tau}^i} = \{ \mathbb{E}_{Y^i \sim \mc{D}^i(\cdot|X=x)}[\tau_j(x, Y^i)] \}_{j \in \mathcal{A}} \).
\end{restatable}

The proof is deferred to the Appendix. Intuitively, Theorem~\ref{theo:consistency} establishes a quantitative link between surrogate and true deferral risk minimization. The first term in the bound captures the impact of optimizing the surrogate objective, while the second term quantifies the approximation error due to holding the predictor \( g \) fixed. Assuming, 
\begin{align*}
&\mathcal{E}_{\Phi^\nu_{\text{def}}}(r) - \mathcal{E}^*_{\Phi^\nu_{\text{def}}}(\mathcal{R}) + \mathcal{U}_{\Phi^\nu_{\text{def}}}(\mathcal{R}) \leq \epsilon_0, \\
& \quad \sum_{i \in \{\s, \e\}} \left( \mathcal{E}_{c_0}(g^i) - \mathcal{E}^B_{c_0}(\mathcal{G}^i) + \mathcal{U}_{c_0}(\mathcal{G}^i) \right) \leq \epsilon_1,
\end{align*}
then the true deferral excess risk satisfies
\begin{equation*}
    \begin{aligned}
            & \mathcal{E}_{\ell_{\text{def}}}(g, r) - \mathcal{E}^B_{\ell_{\text{def}}}(\mathcal{G}, \mathcal{R}) + \mathcal{U}_{\ell_{\text{def}}}(\mathcal{G}, \mathcal{R}) \leq \\
&\,\,\, \epsilon_1 + \left( \sum_{i \in \{\s, \e\}} \mspace{-30mu}\|\boldsymbol{\overline{\tau}^i}\|_1 \right) \Gamma^\nu\left( \frac{\epsilon_0}{\sum_{i \in \{\s, \e\}} \|\boldsymbol{\overline{\tau}^i}\|_1} \right).
    \end{aligned}
\end{equation*}
Importantly, in the realizable regime where both function classes are unrestricted (i.e., \( \mathcal{R} = \mathcal{R}_{\text{all}} \) and \( \mathcal{G} = \mathcal{G}_{\text{all}} \)), we recover exact minimization of both surrogate and true risks.

\begin{restatable}[Bayes-consistency]{corollary}{bayes}
\label{cor:bayes}
Suppose the conditions of Theorem~\ref{theo:consistency} hold and \( \mathcal{R} = \mathcal{R}_{\text{all}}, \mathcal{G} = \mathcal{G}_{\text{all}} \). Then the \sdl{} is Bayes-consistent: for any sequence \( (g_k, r_k) \in \mathcal{G}_{\text{all}} \times \mathcal{R}_{\text{all}} \) with
\[
\mathcal{E}_{\Phi_{\text{def}}^\nu}(r_k) - \mathcal{E}_{\Phi_{\text{def}}^\nu}^\ast(\mathcal{R}_{\text{all}}) \xrightarrow{k \to \infty} 0,
\]
we have $\mathcal{E}_{\ell_{\text{def}}}(g_k, r_k) - \mathcal{E}_{\ell_{\text{def}}}^B(\mathcal{G}_{\text{all}}, \mathcal{R}_{\text{all}}) \xrightarrow{k \to \infty} 0.$
\end{restatable}

\noindent This result confirms that minimizing our surrogate training objective leads to asymptotically optimal deferral policies under general conditions.

\paragraph{Single expert allocation.} 
When both the \Start{} and \End{} span predictions must be delegated to the same agent, we define a unified deferral policy \( \pi^\ast(x) \) that minimizes the total Bayes risk across tasks. This yields the following optimal allocation rule, proved in the appendix:

\begin{restatable}[Bayes-Optimal Deferral Policy]{lemma}{bayesalloc}
\label{optimal_rule}
Given \( x \in \mathcal{X} \), the Bayes-optimal policy assigns the query to the agent \( j \in \mathcal{A} \) minimizing the expected cumulative risk:
\[
\pi^\ast(x) = \argmin_{j \in \mathcal{A}} \sum_{i \in \{\s, \e\}} \eta_j^i(x),
\]
where $\eta_j^i(x)=\mb{E}_{Y^i\sim\mc{D}^i(\cdot|X=x)}[c_j(x,Y^i)]$ denotes the expected cost of assigning task \( i \) to agent \( j \).
\end{restatable}

In practice, we approximate this policy by learning per-task rejector scores \( \hat{r}^{ i}(x, j) \) that estimate \( \eta_j^i(x) \). This gives rise to the learned allocation rule:

\begin{restatable}[Learned allocation policy]{corollary}{learn}
\label{cor:learn}
Assume the per-task surrogate $\Phi_{01}^{\nu}$ is strictly proper
(Lemma~\ref{lemma:surrogate}), so it admits a strictly decreasing
link $\psi\colon[0,1]\!\to\!\mathbb R$.
Suppose further that the excess surrogate risk of the learned
score function $\hat r$ satisfies
\[
  \mathcal E_{\Phi^\nu_{\mathrm{def}}}(\hat r)
  \;-\;
  \mathcal E^{\ast}_{\Phi^\nu_{\mathrm{def}}}(\mathcal R_{\text{all}})
  \;\xrightarrow[n\to\infty]{} 0 .
\]
Then, as the training-sample size \(n\) grows,
\[
  \Pr_{X\sim\mathcal D}\bigl[
     \hat\pi(X)\neq\pi^{\ast}(X)
  \bigr]
  \;\xrightarrow{n\to\infty} 0 ,
\]
so the learned allocation rule
\[
  \hat\pi(x)
  = \arg\max_{j\in\mathcal A}
      \;\sum_{i\in\{\text{start},\text{end}\}}
        \hat r^{\,i}(x,j),
\]
converges in probability to the Bayes-optimal deferral policy
of Lemma~\ref{optimal_rule}.
\end{restatable}

A proof is given in the appendix. Under the consistency guarantees established in Theorem~\ref{theo:consistency}, this converges to the Bayes-optimal allocation \( \pi^\ast(x) \), ensuring reliable and cost-sensitive deferral in the single-agent allocation setting. The query is allocated to an agent according to Algorithm~\ref{algo_inference1}.

\vspace{\baselineskip}
\begin{algorithm}[ht]
\caption{Inference with Corollary \ref{cor:learn}}\label{algo_inference1}
\begin{algorithmic}[1]
   \STATE \textbf{Input:} Query \( x = (q, c)\) with question \( q \) and context \( c \)
   \STATE \textbf{Rejector Evaluation:} Compute scores \( \hat{r}(x) = (\hat{r}^{\s}(x), \hat{r}^{\e}(x)) \)
   \STATE \textbf{Agent Allocation:} Select agent via the learned policy:
   \vspace{-1.5em}
   \[
   \hat{\pi}(x) = \argmax_{j \in \mathcal{A}} \sum_{i \in \{\s, \e\}} \hat{r}^{i}(x, j)
   \]
       \vspace{-2em} 
       \STATE \textbf{Prediction:} 
   \begin{itemize}
       \item \textbf{If} \( \hat{\pi}(x) = 0 \), return prediction \( g(x) \)
       \item \textbf{Else} return expert prediction \( m_{\hat{\pi}(x)}(x) \)
   \end{itemize}
\end{algorithmic}
\end{algorithm}
\vspace{\baselineskip}



%% file: Sections/Experiments.tex
\section{EVALUATION}

\subsection{Motivation}
We evaluate L2D for balancing predictive accuracy and computational efficiency in EQA, focusing on realistic deployment scenarios with lightweight on-device models and stronger remote experts. Our goal is not to surpass state-of-the-art accuracy, but to show that L2D effectively trades off answer quality against computational cost, reliably identifying when deferral is warranted and minimizing reliance on expensive models.

\subsection{Setting}
We test L2D on three standard EQA benchmarks: SQuADv1 \citep{rajpurkar2016squad}, SQuADv2 \citep{rajpurkar2018knowdontknowunanswerable}, and TriviaQA \citep{joshi2017triviaqa}. The training and inference procedures are given in Algorithms~\ref{algo_training1} and~\ref{algo_inference1}.

\paragraph{Agents.}
We evaluate our L2D framework in a deployment-motivated setting with one lightweight general-purpose model and two stronger EQA experts. As the on-device model, we select \textsc{LLaMA-3.2-1B}~\citep{touvron2023llamaopenefficientfoundation}, which offers broad task coverage while remaining computationally viable for resource-constrained environments. For simplicity, we use the base weights without fine-tuning. The expert models, \( M_1 \) and \( M_2 \), are \textsc{ALBERT-Base} and \textsc{ALBERT-XXL}, respectively~\citep{lan2020albertlitebertselfsupervised}. These models are well suited to EQA and achieve high span-selection accuracy, but they do not offer LLaMA's general-purpose capabilities, making them less suitable as the on-device default model. To capture the disparity in inference cost between \( M_1 \) and \( M_2 \), we define a cost ratio \( R = \frac{\text{GFLOPs}(M_2)}{20 \cdot \text{GFLOPs}(M_1)} \) and apply an expert penalty \(\beta_2 = R \beta_1\) to discourage excessive reliance on the more computationally intensive model.

To avoid tying the analysis to deployment-specific factors such as infrastructure or communication stack, we use GFLOPs as a simple proxy when defining this ratio. In practical deployments, users could instead calibrate the \(\beta\) costs using quantities such as network latency or monetary cost.
  

\paragraph{Rejector.}
To enable cost-aware query allocation across agents, we employ a compact model designed for small-device deployment (Figure~\ref{fig:rejector_architecture}). Specifically, we instantiate the rejector using the TinyBERT architecture~\citep{devlin2018bert}, which contains only 4.39M parameters, just 0.35\% the size of the LLaMA-3.2-1B model, making it well suited to low-compute environments. The rejector is trained using the surrogate deferral loss (SDL) introduced in Lemma~\ref{lemma:surrogate}. We adopt the multiclass formulation \(\Phi_{01}^\nu\) with \(\nu = 1\), corresponding to the standard log-softmax loss~\citep{mao2023crossentropylossfunctionstheoretical}.

\paragraph{Benchmark:} We benchmark against vote-based ensembles \citep{10.1023/A:1018054314350, trad2024ensemblenotassessingmajority}, which most closely resemble our setting because they combine multiple models under direct allocation. However, ensembles query all models in parallel and therefore operate under a different computational regime. We also benchmark against a larger model from the Llama-3 family, Llama-3-8B \citep{grattafiori2024llama3herdmodels}, to assess whether our framework can match or exceed the performance of a substantially larger model while retaining the deployment advantages of the 1B variant. When prompting both Llama-3.2-1B and Llama-3-8B, we use few-shot demonstrations \citep{brown2020languagemodelsfewshotlearners}, listed in the appendix.

Additionally, we benchmark our approach against the family of single-expert routers introduced in \citep{ding2024hybridllmcostefficientqualityaware}: the probabilistic router $r_{\text{prob}}$, the deterministic router $r_{\text{deter}}$, and the relaxed router with transformation $r_{\text{trans}}$. We use ALBERT-Base as the model and ALBERT-XXL as the expert. To adapt this framework to EQA, we use $c_j(x,y^i)$ as the quality function in place of BART \citep{yuan2021bartscoreevaluatinggeneratedtext}.


\paragraph{Metrics:} We measure EQA performance using Exact Match (EM). We emphasize that, although the specialist models are stronger, they are not suitable candidates for $g$. We also report the GFLOPs/EM ratio, which measures computational cost per unit of performance, and allocation ratios, which quantify the fraction of queries deferred to experts and therefore reflect how $r$ accounts for cost. Finally, we report true-positive and false-positive rates (TPR/FPR). A true positive occurs when the main model is incorrect and the query is correctly deferred to an accurate expert. A false positive occurs when the query is deferred to an incorrect expert even though the main model is correct.

\subsection{Results}

\paragraph{Sanity Check Against Random Allocation.}
We begin with a comparison against random allocation as a sanity check. Figure~\ref{fig:tdl_random_all} compares the true deferral loss of L2D with that of a random allocation policy across all benchmark datasets. This is the most direct way to verify that the rejector is learning a non-trivial allocation rule rather than merely redistributing queries across agents.

\begin{figure}[t]
\centering
\begin{subfigure}[t]{0.48\linewidth}
\centering
\includegraphics[width=\linewidth]{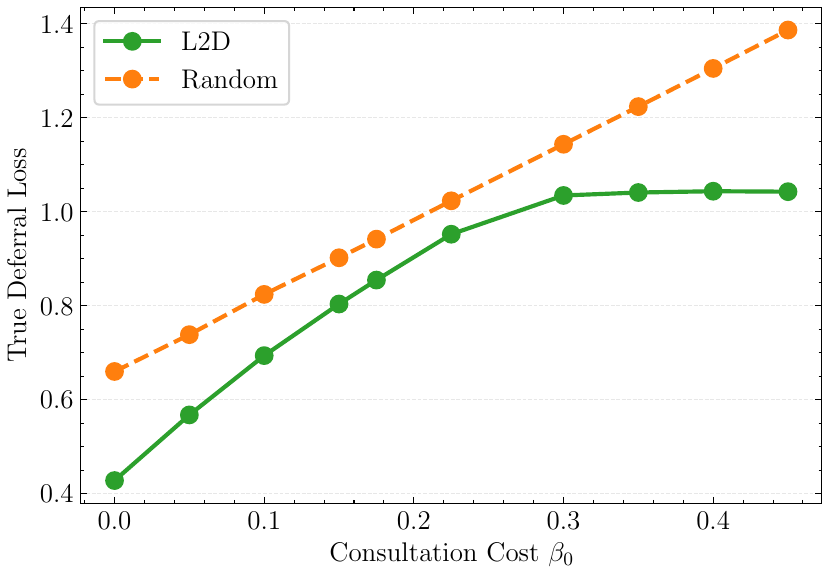}
\caption{SQuADv1}
\end{subfigure}
\hfill
\begin{subfigure}[t]{0.48\linewidth}
\centering
\includegraphics[width=\linewidth]{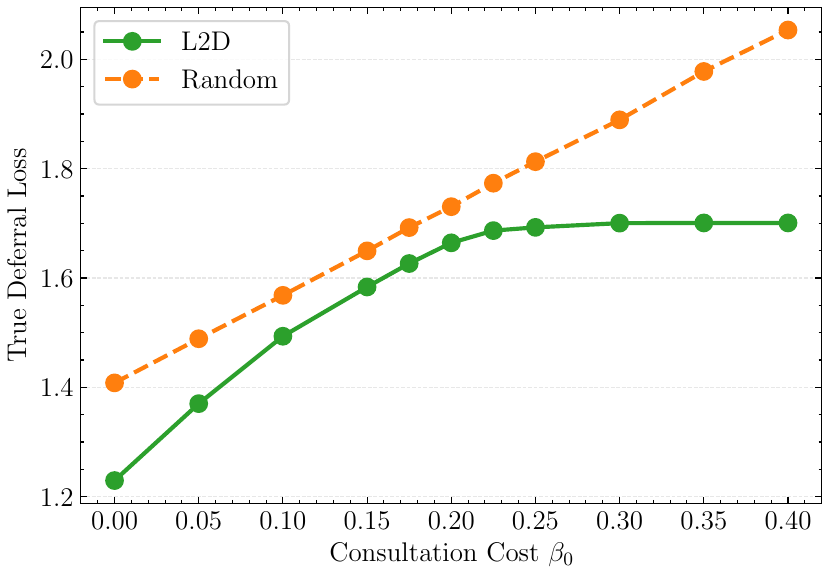}
\caption{SQuADv2}
\end{subfigure}

\vspace{0.2em}

\makebox[\linewidth][c]{%
\begin{subfigure}[t]{0.48\linewidth}
\centering
\includegraphics[width=\linewidth]{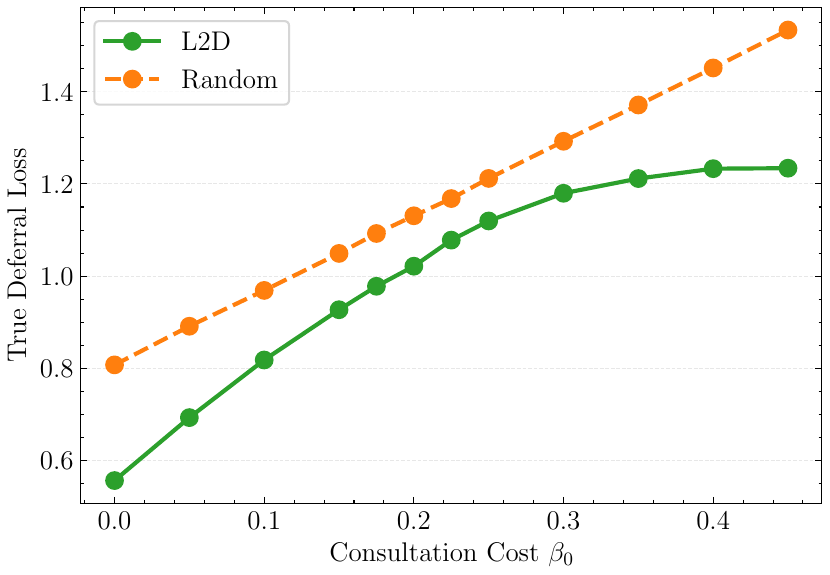}
\caption{TriviaQA}
\end{subfigure}}
\caption{\textbf{TDL against random allocation as consultation cost \(\beta_0\) varies.}}
\label{fig:tdl_random_all}
\end{figure}

For all values of \(\beta_0\), L2D achieves substantially lower TDL, demonstrating that it defers selectively rather than arbitrarily. Because TDL penalizes both unnecessary expert usage and incorrect on-device predictions, the consistent gap over random allocation indicates that the learned policy captures the intended cost-quality trade-off.

\paragraph{Cost-Aware Allocation Behavior.}
We next examine how the learned policy responds to the consultation cost. Across all three datasets, increasing the consultation penalty systematically reduces reliance on expensive experts and shifts traffic toward cheaper agents. In particular, increasing $\beta_1$ moves allocation toward the cheaper ALBERT-Base expert, while extreme values of \(\beta_1\) can suppress allocation to either \(g\) or \(M_2\). This behavior indicates that the rejector is not merely optimizing accuracy, but is explicitly responding to the specified cost structure. Additional plots of normalized cost and expert query rate are provided in the appendix.

\paragraph{Accuracy-Efficiency Trade-Offs.}
The next question is whether these cost-aware routing decisions preserve answer quality. Figure~\ref{fig:accuracy_compute_tradeoffs} reports the evolution of exact match and GFLOPs/EM as the consultation cost varies. These plots make the performance-efficiency trade-off explicit and clarify where the learned policy remains competitive while reducing computation.

\begin{figure*}[t]
\centering
\setlength{\abovecaptionskip}{2pt}
\setlength{\belowcaptionskip}{0pt}

\begin{subfigure}[t]{\textwidth}
\centering
\includegraphics[width=0.315\textwidth]{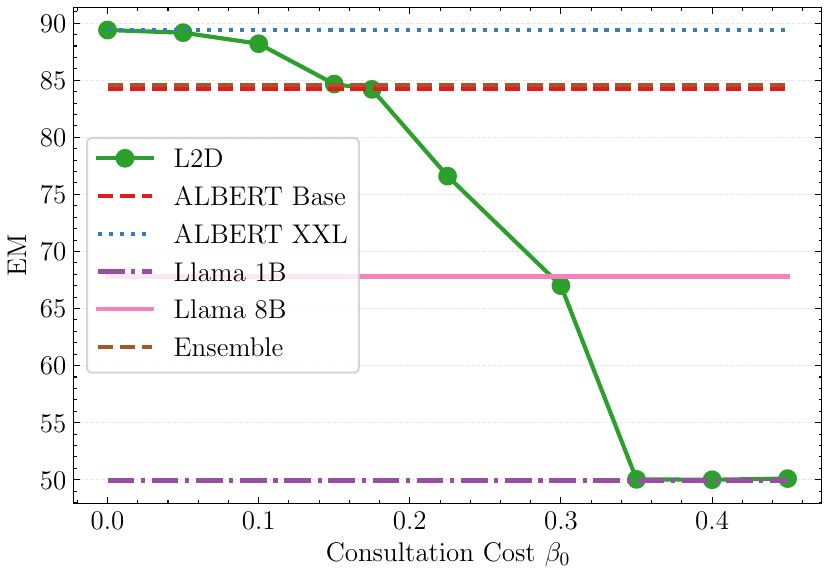}
\includegraphics[width=0.315\textwidth]{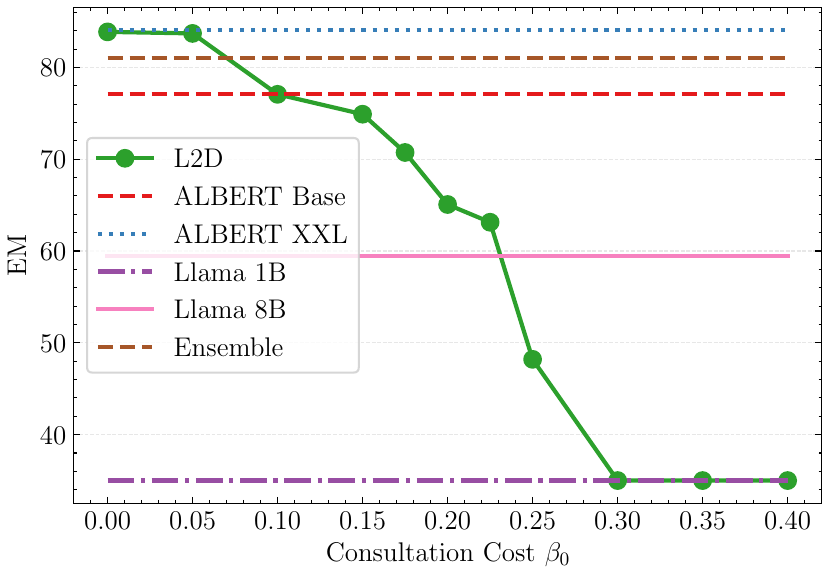}
\includegraphics[width=0.315\textwidth]{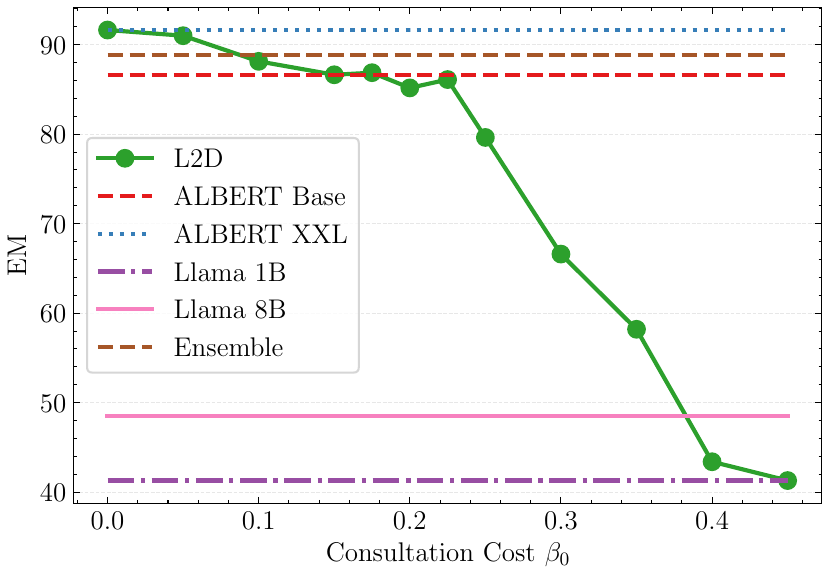}
\caption{Exact match (EM) as consultation cost \(\beta_0\) varies.}
\end{subfigure}

\vspace{0.2em}

\begin{subfigure}[t]{\textwidth}
\centering
\includegraphics[width=0.315\textwidth]{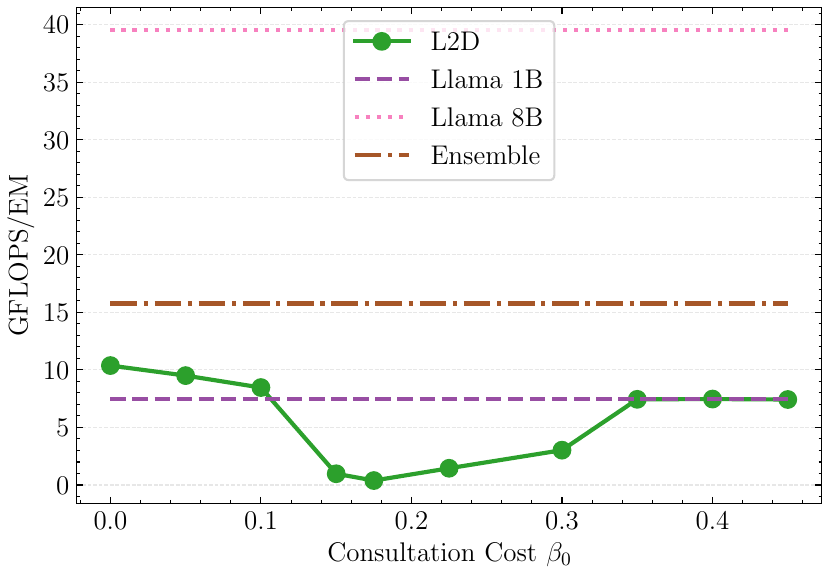}
\includegraphics[width=0.315\textwidth]{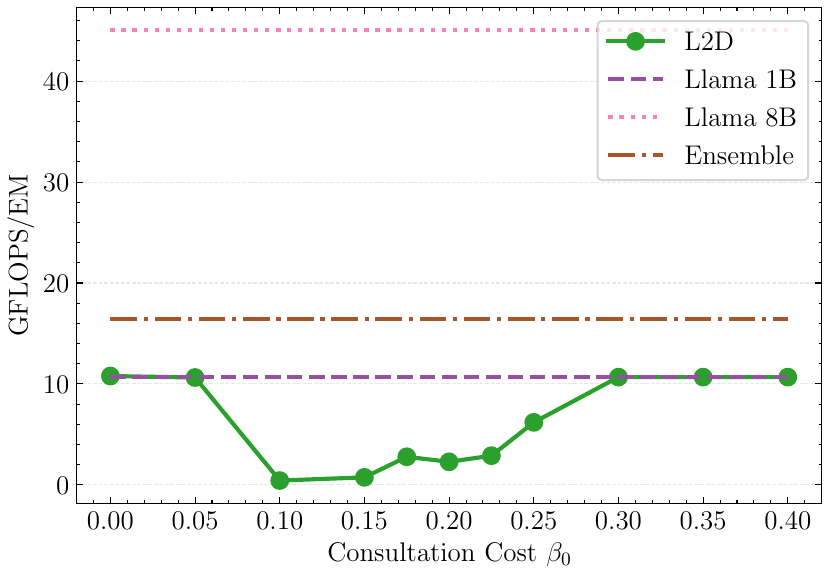}
\includegraphics[width=0.315\textwidth]{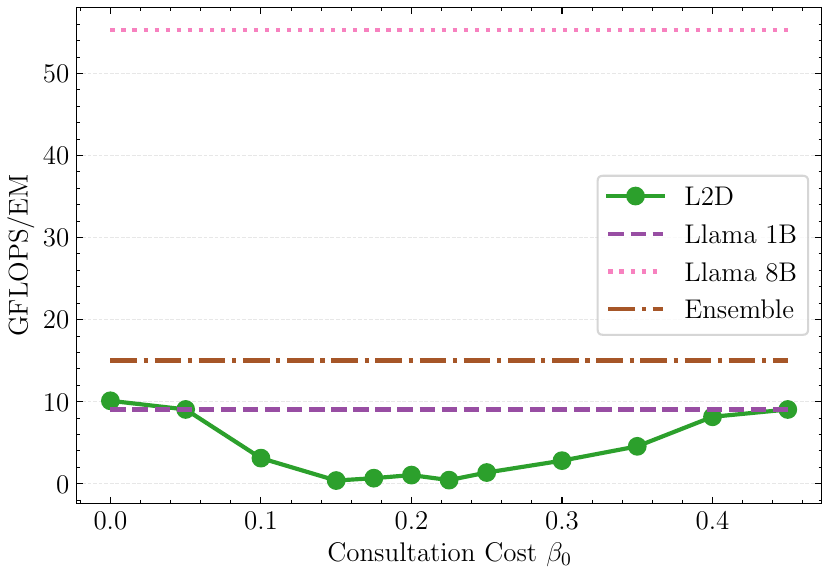}
\caption{GFLOPs/EM as consultation cost \(\beta_0\) varies.}
\end{subfigure}

\captionsetup{skip=0.8\baselineskip}
\caption{
\textbf{Accuracy-efficiency trade-offs across datasets.}
Columns correspond to SQuADv1, SQuADv2, and TriviaQA from left to right. The top row reports exact match, and the bottom row reports computational cost per unit of performance.
}
\label{fig:accuracy_compute_tradeoffs}
\end{figure*}

As \(\beta_0\) increases, the policy is encouraged to use cheaper agents more aggressively, which leads to a gradual reduction in EM, as shown in the top row of Figure~\ref{fig:accuracy_compute_tradeoffs}. This behavior is expected: larger consultation penalties make aggressive expert usage less attractive. The key point is that the resulting loss in exact match is controlled rather than abrupt, which indicates that the rejector is making meaningful trade-offs instead of collapsing to a single fixed agent.

The bottom row of Figure~\ref{fig:accuracy_compute_tradeoffs} shows that, across all benchmark datasets, our approach uses the fewest computational resources per unit of EM. This highlights the cost-performance efficiency of the proposed method. We also obtain at least a 4x efficiency improvement relative to naively running the larger, more expensive Llama 8B on the edge device while maintaining comparable performance. This underscores the practical value of selective expert involvement.

\paragraph{Single-Expert Router Comparison.}
Finally, we compare our framework with single-expert routing baselines adapted from \citet{ding2024hybridllmcostefficientqualityaware}. This comparison isolates the effect of the deferral objective from the multi-expert setting and tests whether our loss remains beneficial in the simpler two-agent case.

\begin{table}[h]
\centering 
\caption{Single-Expert Setting, Exact Match on SQuADv1, L2D against Query Routers \citep{ding2024hybridllmcostefficientqualityaware}}
\begin{tabular}{@{}cccccc@{}}
\toprule
Cost $\beta_0$ & Our & $r_{\text{prob}}$ & $r_{\text{deter}}$ & $r_{\text{trans}}$ \\
\midrule
$0.0$ &$ 84.2$ & $ 84.2$ & $ 84.2$ & $ 84.2$    \\
\bottomrule
\end{tabular}
\label{tab:em}
\end{table}

\begin{table}[h]
\centering
\caption{Single-Expert Setting, Exact Match / Cost on SQuADv1, L2D against Query Routers \citep{ding2024hybridllmcostefficientqualityaware}}
\begin{tabular}{@{}cccccc@{}}
\toprule
Cost $\beta_0$ & Our & $r_{\text{prob}}$ & $r_{\text{deter}}$ & $r_{\text{trans}}$  \\
\midrule
$0.25$ &$12.16$ & $ 3.36 $ & $ 3.37 $ & $ 3.36$ \\
\midrule
$0.5$ &$ 3.04$ & $ 1.68 $ & $ 1.68$ & $ 1.68$ \\
\bottomrule
\end{tabular}
\label{tab:empercost}
\end{table}

Table~\ref{tab:em} shows that, in the zero-cost single-expert setting, our method matches the performance of existing query routers. Table~\ref{tab:empercost} then shows that, once expert consultation incurs nonzero cost, our method outperforms the router baselines. This suggests that the proposed objective is better aligned with discontinuous cost structures of the form $c_j(x,y^i)$ and therefore remains effective even when the routing problem becomes explicitly cost-sensitive.

%% file: Sections/Conclusion.tex
\section{CONCLUSION}

We introduced a Learning-to-Defer framework for EQA that dynamically allocates queries to the most suitable agent while balancing accuracy and computational cost. The framework is supported by theoretical guarantees and is intended for resource-constrained deployment settings. Empirical evaluations on SQuADv1, SQuADv2, and TriviaQA show improved reliability-efficiency trade-offs relative to larger LLMs, ensembles, and routing baselines.

\section*{ACKNOWLEDGMENTS}
This research was supported by the National Research Foundation, Singapore, under its AI Singapore Programme (AISG Award No.\ AISG2-PhD-2023-01-041-J) and by A*STAR, and forms part of the DesCartes programme, which is supported by the National Research Foundation, Prime Minister's Office, Singapore, under its Campus for Research Excellence and Technological Enterprise (CREATE) programme.

%% file: Sections/Appendix.tex
\clearpage
\onecolumn
\appendix 

\section{APPENDIX}
\subsection{Current Approaches} 

\begin{figure}[H] 
    \centering
    \begin{tikzpicture}[scale=0.8]
        \draw[fill=blue!20] (-2.2,0.6) rectangle (-1,1.1) node[pos=.5, scale=0.7] {LLM$_1$};
        \draw[fill=red!20] (-0.2,0.6) rectangle (1,1.1) node[pos=.5, scale=0.7] {LLM$_2$};
        \draw[fill=green!20] (1.8,0.6) rectangle (3,1.1) node[pos=.5, scale=0.7] {LLM$_3$};
        \draw[thick, ->] (-1,0.85) -- (-0.2,0.85); 
        \draw[thick, ->] (1,0.85) -- (1.8,0.85); 
    \end{tikzpicture}
    \hspace{1cm}  
    \begin{tikzpicture}[scale=1.1]
        \draw[fill=brown!20] (1,0) rectangle (2.2,0.5) node[pos=.5, scale=0.7] {Router};
        \draw[fill=blue!20] (0,1.2) rectangle (1.2,1.7) node[pos=.5, scale=0.7] {LLM$_{weak}$};
        \draw[fill=red!20] (2,1.2) rectangle (3.2,1.7) node[pos=.5, scale=0.7] {LLM$_{strong}$};
        \draw[thick, ->] (1.6,0.5) -- (0.6,1.2);
        \draw[thick, ->] (1.6,0.5) -- (2.6,1.2);
    \end{tikzpicture}
    \hspace{1cm}  
    \begin{tikzpicture}[scale=0.8]
        \draw[fill=brown!20] (1.25,0) rectangle (2.45,0.5) node[pos=.5, scale=0.7] {Router};
        \draw[fill=blue!20] (-0.5,1.2) rectangle (0.7,1.7) node[pos=.5, scale=0.7] {LLM};
        \draw[fill=red!20] (1.25,1.2) rectangle (2.45,1.7) node[pos=.5, scale=0.7] {M$_1$};
        \draw[fill=green!20] (3,1.2) rectangle (4.2,1.7) node[pos=.5, scale=0.7] {M$_2$};
        \draw[thick, ->] (1.85,0.5) -- (0.1,1.2);
        \draw[thick, ->] (1.85,0.5) -- (1.85,1.2);
        \draw[thick, ->] (1.85,0.5) -- (3.6,1.2);
    \end{tikzpicture}
    \captionsetup{skip=2\baselineskip}
    \caption{From left to right: model cascades, query routing, and Learning-to-Defer (ours). Our framework retains the multi-model nature of cascades while allowing direct allocation, as in query routing.}
    \label{currentapproaches}
\end{figure}

\subsection{Proof Lemma \ref{lemma:rejector}}\label{proof:rejector}

\rejector* 
\begin{proof}
For \(i\in\{\text{start},\text{end}\}\) and any input \(x\in\mathcal X\), define
\[
\eta_0^{i}(x)
  \;=\;
  \Pr_{(X,Y^{i})\sim\mathcal D^{i}}
     \!\bigl[g^{i}(x)\neq Y^{i}\mid X=x\bigr],
\]

\[
\eta_{j}^{i}(x)
  \;=\;
  \alpha_{j}\,
  \Pr_{(X,Y^{i})\sim\mathcal D^{i}}
     \!\bigl[m_{j}^{i}(x)\neq Y^{i}\mid X=x\bigr]
  \;+\;
  \beta_{j},
  \qquad j=1,\dots,J,
\]
where \(\alpha_{j}\ge0\) and \(\beta_{j}\ge0\) are the per-expert error
penalty and fixed consultation cost, respectively. \\

\noindent We analyze the conditional risk associated with the \tdl{} for a fixed task \( i \in \{\s, \e\} \). For a given input \( x \in \mathcal{X} \), this risk is defined as the expected cost incurred when using a rejector \( r^i \in \mathcal{R}^i \) to select an agent:
\begin{align*}
\mathcal C_{\ell_{\text{def}}}^{i}(g^{i},r^{i},x)
  &= \mathbb E_{Y^{i}\sim\mathcal D^{i}(\,\cdot\mid X=x)}
     \Bigl[
       \sum_{j=0}^{J}
         c_{j}\bigl(x,Y^{i}\bigr)\,
         \mathbf 1\!\{r^{i}(x)=j\}
     \Bigr] \\[2pt]
  &= \eta_{0}^{i}(x)\,\mathbf 1\!\{r^{i}(x)=0\}
     \;+\;
     \sum_{j=1}^{J}
       \eta_{j}^{i}(x)\,
       \mathbf 1\!\{r^{i}(x)=j\}.
\end{align*}
To determine the optimal rejection strategy, we minimize this conditional risk over all possible choices of \( g^i \in \mathcal{G}^i \) and \( r^i \in \mathcal{R}^i \):
\begin{align*}
    \inf_{r^i \in \mathcal{R}^i,\, g^i \in \mathcal{G}^i} 
    \mathcal{C}_{\ell_{\text{def}}}^i(g^i, r^i, x)
    &= \min \left\{
        \inf_{g^i \in \mathcal{G}^i} \eta_0^i(x),
        \min_{j \in [J]} \eta_j^i(x)
    \right\}.
\end{align*}

This implies the structure of the Bayes-optimal rejector \( r^{B,i} \), which assigns the query to the agent with minimal conditional risk:
\[
r^{B,i}(x) = 
\begin{cases}
0, & \text{if } \displaystyle \inf_{g^i \in \mathcal{G}^i} \eta_0^i(x) 
\leq \min_{j \in [J]} \eta_j^i(x), \\
\argmin_{j \in [J]} \eta_j^i(x), & \text{otherwise}.
\end{cases}
\]
Thus, the Bayes-rejector minimizes the expected deferral loss by comparing the cost-adjusted risks of the main model and the experts, selecting the most reliable agent accordingly.
\end{proof}

\subsection{Proof Lemma \ref{lemma:surrogate}} 
\label{proof_surrogate}

\surrogate*

\begin{proof}
We consider a unified agent space \( \mathcal{A} = \{0, 1, \dots, J\} \), where index \( 0 \) denotes the main model and indices \( j \in [J] \) denote experts. For each query–label pair \( (x,y^i) \), define the agent-specific cost as \( c_j(x,y^i) \geq 0 \) for all \( j \in \mathcal{A} \). Let the total consultation cost be:
\[
C_{\text{tot}}^i(x,y^i) = \sum_{j \in \mathcal{A}} c_j(x,y^i).
\]
Define the deferral complement cost (excluding agent \( j \)) as
\[
\tau_j^i(x,y^i) = \sum_{\substack{q \in \mathcal{A} \\ q \neq j}} c_q(x,y^i) = C_{\text{tot}}^i(x,y^i) - c_j(x,y^i).
\]

Now, for a selector \( r^i: \mathcal{X} \to \mathcal{A} \), the deferral loss component is
\[
\ell_{\text{def}}^i(r^i(x), y^i) = c_{r^i(x)}(x,y^i).
\]
Because \( \mathcal{A} \setminus \{r^i(x)\} \) is the complement of the selected agent, we can equivalently write:
\begin{align}
\ell_{\text{def}}^i(r^i(x), y^i)
&= C_{\text{tot}}^i(x,y^i) - \sum_{j \neq r^i(x)} c_j(x,y^i) \nonumber \\
&= \sum_{j \in \mathcal{A}} \tau_j^i(x,y^i)  \mathbf{1}\{r^i(x) \neq j\} - (J-1)  C_{\text{tot}}^i(x,y^i). \label{eq:deferral_top1}
\end{align}

Now aggregate the deferral loss over both span components \( i \in \{\s, \e\} \) to obtain the total deferral loss:
\[
\sum_{i \in \{\s, \e\}} \ell_{\text{def}}^i(r^i(x), y^i)
= \sum_{i \in \{\s, \e\}} \left[
\sum_{j \in \mathcal{A}} \tau_j^i(x,y^i)  \mathbf{1}\{r^i(x) \neq j\}
- (J - 1)  C_{\text{tot}}^i(x,y^i)
\right].
\]

Let \( \Phi_{01}^{u}(r^i, x, j) \) be a surrogate loss that upper bounds the 0-1 indicator:
\[
\mathbf{1}\{r^i(x) \neq j\} \leq \Phi_{01}^{u}(r^i, x, j), \quad \forall j \in \mathcal{A}.
\]
This holds for the cross-entropy surrogate and other upper-bounding families \citep{mao2023crossentropylossfunctionstheoretical}. Since each \( \tau_j^i(x,y^i) \geq 0 \), the total loss satisfies:
\[
\sum_{i \in \{\s, \e\}} \ell_{\text{def}}^i(r^i(x), y^i)
\leq \sum_{i \in \{\s, \e\}} \left[
\sum_{j \in \mathcal{A}} \tau_j^i(x,y^i)  \Phi_{01}^{u}(r^i, x, j)
- (J - 1)  C_{\text{tot}}^i(x,y^i)
\right].
\]

Furthermore, as the term $(J - 1)  C_{\text{tot}}^i(x,y^i)$ 
does not depend on $r$, we can formalize the following surrogate loss
\begin{equation}
    \Phi(r,x,y) =  \sum_{i \in \{\s, \e\}} 
\sum_{j \in \mathcal{A}} \tau_j^i(x,y^i)  \Phi_{01}^{u}(r^i, x, j)
\end{equation}
\end{proof}

\subsection{Proof Theorem \ref{theo:consistency}}\label{proof:consistency}
\theoconsistency*
\begin{proof}
The proof of Theorem~\ref{theo:consistency} uses the following lemma, which states the relevant consistency property for a general distribution.

\begin{restatable}[$\mc{R}^i$-consistency bound]{lemma}{cons}\label{proof:cons}
Fix an input $x\in\mc{X}$ and any distribution $\mc{D}$. Suppose there exists a non-decreasing, concave function $\Gamma^\nu: \mb{R}^+ \to \mb{R}^+$ for $\nu\geq 0$ such that the \( \mathcal{R}^i \)-consistency bounds hold for any distribution \( \mathcal{D} \):
\begin{equation*}
\begin{aligned}
    & \mathcal{E}_{\Phi_{01}^\nu}(r^i) - \mathcal{E}^*_{\Phi_{01}^\nu}(\mathcal{R}^i) + \mathcal{U}_{\Phi_{01}^\nu}(\mathcal{R}^i) \geq \Gamma^\nu( \mathcal{E}_{\ell_{01}}(r^i) -  \mathcal{E}^B_{\ell_{01}}(\mathcal{R}^i) + \mathcal{U}_{\ell_{01}}(\mathcal{R}^i)),
\end{aligned}
\end{equation*}
or, equivalently, for $\boldsymbol{p^i} \in \Delta^{|\mc{A}|}$,
\begin{equation*}
    \sum_{j\in\mc{A}} p_j^i1_{\{r^i(x)\neq j\}} - \inf_{r^i\in\mc{R}^i}\sum_{j\in\mc{A}} p_j^i1_{\{r^i(x)\neq j\}} \leq \Gamma^{\nu}\Big( \sum_{j\in\mc{A}} p_j^i\Phi_{01}^\nu(r^i,x,j) - \inf_{r^i\in\mc{R}^i}\sum_{j\in\mc{A}} p_j^i\Phi_{01}^\nu(r^i,x,j) \Big)
\end{equation*}
\end{restatable}
\noindent Define, for each $j\in\mc{A}=\{0, \dots, J\}$, the cost
\[
\overline c_{j}^{\,i,\ast}(x)
\;=\;
\begin{cases}
\inf_{g^{i}\in\mathcal G^{i}}
      \mathbb E_{Y^{i}\sim\mathcal D^{i}(\,\cdot\mid X=x)}
      \bigl[c_{0}\!\bigl(x,Y^{i}\bigr)\bigr]
\\[8pt] \hfill  
   =\;
   \inf_{g^{i}\in\mathcal G^{i}}
        \Pr_{Y^{i}\sim\mathcal D^{i}(\,\cdot\mid X=x)}
        \bigl[g^{i}(x)\neq Y^{i}\bigr],
& j=0,
\\[16pt]
\mathbb E_{Y^{i}\sim\mathcal D^{i}(\,\cdot\mid X=x)}
      \bigl[c_{j}\!\bigl(x,Y^{i}\bigr)\bigr]
\\[8pt] \hfill
   =\; \alpha_{j}\,
        \Pr_{Y^{i}\sim\mathcal D^{i}(\,\cdot\mid X=x)}
           \bigl[m_{j}^{i}(x)\neq Y^{i}\bigr]
     + \beta_{j},
& j=1,\dots,J.
\end{cases}
\]

Recall the previously established result from Lemma~\ref{lemma:rejector}:
\begin{equation*}
    \begin{aligned}
       \mc{C}_{\ell_{\text{def}}}^{\ast,i}(g^i, r^i,x) & = \min\Big\{ \inf_{g^{i}\in\mathcal G^{i}}
        \Pr_{Y^{i}\sim\mathcal D^{i}(\,\cdot\mid X=x)}
        \bigl[g^{i}(x)\neq Y^{i}\bigr], \alpha_{j}\,
        \Pr_{Y^{i}\sim\mathcal D^{i}(\,\cdot\mid X=x)}
           \bigl[m_{j}^{i}(x)\neq Y^{i}\bigr]
     + \beta_{j}\Big\} \\
       & = \min_{j\in\mc{A}} \overline{c}_j^{\,i,\ast}(x)
    \end{aligned}
\end{equation*}
Therefore, we introduce the calibration gap $\Delta\mc{C}^i_{\ell_{\text{def}}}:= \mc{C}_{\ell_{\text{def}}}^{i} - \mc{C}_{\ell_{\text{def}}}^{\ast,i}$:
\begin{equation}
\begin{aligned}
    \Delta\mc{C}^i_{\ell_{\text{def}}}(r^i, g^i,x) & = \mc{C}_{\ell_{\text{def}}}^{i}(r^i, g^i,x) - \min_{j\in\mc{A}} \overline{c}_j^{\,i,\ast}(x)\\
    & = \mc{C}_{\ell_{\text{def}}}^{i}(r^i, g^i,x) - \min_{j\in\mc{A}}\overline{c}_j^i(x) + \Big(\min_{j\in\mc{A}}\overline{c}_j^i(x) - \min_{j\in\mc{A}} \overline{c}_j^{\,i,\ast}(x)\Big)  
\end{aligned}
\end{equation}

\noindent Define the first term as $A=\mc{C}_{\ell_{\text{def}}}^{i} - \min_{j\in\mc{A}}\overline{c}_j^i$ and the second term as $B= \min_{j\in\mc{A}}\overline{c}_j^i - \min_{j\in\mc{A}} \overline{c}_j^{\,i,\ast}$, so that $\Delta\mc{C}^i_{\ell_{\text{def}}}=A+B$. By Lemma~\ref{lemma:surrogate},
\begin{equation}
    \min_{j\in\mc{A}}\overline{c}_j^i(x) = \inf_{r^i\in\mc{R}}\mc{C}^i_{\ell_{def}}(r^i, g^i,x) = \inf_{r^i\in\mc{R}}\sum_{j\in\mc{A}} \overline{\tau}_j^i(x)1_{\{r^i(x)\neq j\}} 
\end{equation}
It follows by definition of the conditional risk:
\begin{equation}
    \begin{aligned}
        A = \sum_{j\in\mc{A}} \overline{\tau}_j^{i}(x) 1_{\{r^i(x)\neq j\}} - \inf_{r^i\in\mc{R}}\sum_{j\in\mc{A}} \overline{\tau}^i_j(x) 1_{\{r^i(x)\neq j\}}
    \end{aligned}
\end{equation}
We normalize the cost vector \( \boldsymbol{\overline{\tau}^i} \) using the \( \ell_1 \)-norm:

\begin{equation}
\boldsymbol{p}^i = \frac{\boldsymbol{\overline{\tau}^i}}{\|\boldsymbol{\overline{\tau}^i}\|_1} \in \Delta^{|\mathcal{A}|},
\end{equation}
where \( \|\boldsymbol{\overline{\tau}^i}\|_1 \) denotes the \( \ell_1 \)-norm, 
ensuring that \( \boldsymbol{p}^i \) lies within the probability simplex $\Delta^{|\mathcal{A}|} = \left\{ \boldsymbol{p}^i \in \mathbb{R}^{|\mathcal{A}|} \mid p^i_j \geq 0, \sum_{j} p^i_j = 1 \right\}$. Then,

\begin{equation}
    \begin{aligned}
        A & = \|\boldsymbol{\overline{\tau}^i}\|_1\Big(\sum_{j\in\mc{A}} p_j^{i} 1_{\{r^i(x)\neq j\}} - \inf_{r^i\in\mc{R}}\sum_{j\in\mc{A}} p_j^i1_{\{r^i(x)\neq j\}}\Big) \\
        & \leq \|\boldsymbol{\overline{\tau}^i}\|_1 \Gamma^\nu \Big(\sum_{j\in\mc{A}} p_j^{i} \Phi_{01}^{\nu}(r^i, x, j) - \inf_{r^i\in\mc{R}}\sum_{j\in\mc{A}} p_j^i\Phi_{01}^{\nu}(r^i, x, j)\Big) \quad \text{(using Lemma \ref{proof:cons})}\\
        & = \|\boldsymbol{\overline{\tau}^i}\|_1 \Gamma^\nu \Big( \frac{1}{\|\boldsymbol{\overline{\tau}^i}\|_1} \Big[ \sum_{j\in\mc{A}} \overline{\tau}^i_j(x) \Phi_{01}^{\nu}(r^i, x, j) - \inf_{r^i\in\mc{R}}\sum_{j\in\mc{A}} \overline{\tau}^i_j(x) \Phi_{01}^{\nu}(r^i, x, j) \Big]\Big)\\
        & = \|\boldsymbol{\overline{\tau}^i}\|_1 \Gamma^\nu \Big( \frac{\Delta\mc{C}_{\text{def}}^i(r^i, x)}{\|\boldsymbol{\overline{\tau}^i}\|_1} \Big) 
    \end{aligned}
\end{equation}
Now, we have the following relationship:
\begin{equation}\label{eq:inequality}
    \begin{aligned}
        B = \min_{j\in\mc{A}}\overline{c}_j^i(x) - \min_{j \in \mc{A}} \overline{c}^{i,\ast}_{j}(x) \leq  \mb{E}_{Y^i \sim \mc{D}(\cdot|X=x)} [c_0(x, Y^i)] - \inf_{g^i \in \mathcal{G}^i} \mb{E}_{Y^i \sim \mc{D}(\cdot|X=x)} [c_0(x, Y^i)]
    \end{aligned}
\end{equation}
Substituting the bound on $B$, we obtain:
\begin{equation}
    \begin{aligned}
        \Delta\mc{C}^i_{\ell_{\text{def}}}(r^i,g^i, x) 
    \leq \|\boldsymbol{\overline{\tau}^i}\|_1 \Gamma^\nu \Big( \frac{\Delta\mc{C}_{\text{def}}^i(r^i)}{\|\boldsymbol{\overline{\tau}^i}\|_1} \Big)  + \mb{E}_{Y^i \sim \mc{D}(\cdot|X=x)} [c_0(x, Y^i)] - \inf_{g^i \in \mathcal{G}^i} \mb{E}_{Y^i \sim \mc{D}(\cdot|X=x)} [c_0(x, Y^i)] 
    \end{aligned}
\end{equation}
Summing over \(i\), we obtain
\begin{equation}
    \begin{aligned}
         \Delta\mc{C}_{\ell_{\text{def}}}(r,g, x) \leq \sum_{i\in\{\s,\e\}} \Big[ \|\boldsymbol{\overline{\tau}^i}\|_1 \Gamma^\nu \Big( \frac{\Delta\mc{C}_{\text{def}}^i(r^i)}{\|\boldsymbol{\overline{\tau}^i}\|_1} \Big)  + \mb{E}_{Y^i \sim \mc{D}(\cdot|X=x)} [c_0(x, Y^i)] - \inf_{g^i \in \mathcal{G}^i} \mb{E}_{Y^i \sim \mc{D}(\cdot|X=x)} [c_0(x, Y^i)]\Big]
    \end{aligned}
\end{equation}
Using the fact that the function $\Gamma$ is concave and that the \Start{} and \End{} are conditionally independent given $x$:
\begin{equation}\label{eq:boxed}
\begin{aligned}
\Delta \mathcal{C}_{\ell_{\text{def}}}(r,g,x) & \leq \Biggl(\sum_{i\in\{\s, \e\}} \|\boldsymbol{\overline{\tau}^i}\|_1\Biggr)
\Gamma^\nu\Biggl( \frac{ \Delta \mathcal{C}_{\text{def}}(r) }{\sum_{i\in\{\s, \e\}} \|\boldsymbol{\overline{\tau}^i}\|_1} \Biggr) \\
& + \sum_{i\in\{\s,\e\}} \left[\mb{E}_{Y^i \sim \mc{D}(\cdot|X=x)} [c_0(x, Y^i)] - \inf_{g^i\in\mathcal{G}^i} \mb{E}_{Y^i \sim \mc{D}(\cdot|X=x)} [c_0(x, Y^i)]\right] 
\end{aligned}
\end{equation}

\noindent Taking expectation with respect to \(X\) yields the excess-risk form $\mb{E}_{X\sim\mc{D}_X}[\Delta\mc{C}_\ell]:=\mc{E}_{\ell} - \mc{E}_{\ell}^B + \mc{U}_\ell$, which gives the desired result:

\begin{equation}
\begin{aligned}
    \mathcal{E}_{\ell_{\text{def}}}(g, r ) - \mathcal{E}^B_{\ell_{\text{def}}}( \mathcal{G}, \mathcal{R}) + \mathcal{U}_{\ell_{\text{def}}}( \mathcal{G}, \mathcal{R}) & \leq  \overline{\Gamma}^\nu\left(\mathcal{E}_{\Phi^\nu_{\text{def}}}(r) - \mathcal{E}^*_{\Phi^\nu_{\text{def}}}(\mathcal{R}) + \mathcal{U}_{\Phi^\nu_{\text{def}}}(\mathcal{R})\right) \\
    & \quad + \sum_{i\in\{ \s, \e\}} \Big( \mathcal{E}_{c_0}(g^i) - \mathcal{E}^B_{c_0}(\mathcal{G}^i) + \mathcal{U}_{c_0}(\mathcal{G}^i)\Big),
\end{aligned}
\end{equation}
with $\overline{\Gamma}^\nu(u) = \Big(\sum_{i\in\{\s, \e\}} \|\boldsymbol{\overline{\tau}^i}\|_1\Big)\Gamma^\nu\Big(\frac{u}{\sum_{i\in\{\s, \e\}} \|\boldsymbol{\overline{\tau}^i}\|_1}\Big)$ and from \citet{mao2023crossentropylossfunctionstheoretical},  it follows for $\nu\geq0$ the inverse transformation $\Gamma^{\nu}(u) = \mathcal{T}^{-1, \nu}(u)$:

\[
\mathcal{T}^{\nu}(u) =
\begin{cases}
\frac{2^{1-\nu}}{1-\nu} \left[ 1 - \left( \frac{(1+u)^{\frac{2-\nu}{2}} + (1-u)^{\frac{2-\nu}{2}}}{2} \right)^{2-\nu} \right] & \nu \in [0,1) \\[12pt]
\frac{1+u}{2} \log[1+u] + \frac{1-u}{2} \log[1-u] & \nu = 1 \\[12pt]
\frac{1}{(\nu-1)n^{\nu-1}} \left[ \left( \frac{(1+u)^{\frac{2-\nu}{2}} + (1-u)^{\frac{2-\nu}{2}}}{2} \right)^{2-\nu} -1 \right] & \nu \in (1,2) \\[12pt]
\frac{1}{(\nu-1)n^{\nu-1}} u & \nu \in [2,+\infty).
\end{cases}
\]

\end{proof}

\subsection{Proof Lemma \ref{optimal_rule}} \label{label_bayes}

\bayesalloc*

\begin{proof}
Fix \(x\in\mathcal{X}\).  
Because \(\s\) and \(\e\) are conditionally independent given \(X=x\), the conditional expected
true-deferral cost when always deferring \emph{both} sub-tasks to agent
\(j\) is the \emph{additive} conditional risk
\(
\mc{C}_{j}(x)=\eta_{j}^{\s}(x)+\eta_{j}^{\e}(x).
\)
Let \(\Pi\) be the space of all measurable single-agent policies
\(\pi:\mathcal{X}\to\mathcal{A}\).
For every \(\pi\in\Pi\) the conditional expected loss at \(x\) is
\(\mc{C}_{\pi(x)}(x)\).
Hence
\(
\displaystyle\inf_{\pi\in\Pi} \mc{C}_{\pi(x)}(x)=\min_{j\in\mathcal{A}} \mc{C}_{j}(x),
\)
and the minimiser is any argument of that minimum.  
Defining \(\pi^{\ast}(x)\) exactly as the \(\argmin\) above
therefore achieves the Bayes (i.e.\ pointwise minimum) conditional risk,
which completes the proof.
\end{proof}

\subsection{Proof Corollary \ref{cor:learn}}\label{label_learn}

\learn*

\begin{proof}
After training, each task-specific rejector outputs \emph{scores}
\(
   \hat r^{\,i}(x,j)\in\mathbb{R},~i\in\{\s,\e\},~j\in\mathcal A,
\)
and the inference-time decision for the two tasks jointly is
\[
   \hat\pi(x)\;=\;\argmax_{j\in\mathcal A}
                  \hat S_j(x),\qquad
   \hat S_j(x)=\sum_{i\in\{\s,\e\}}\hat r^{\,i}(x,j).
\]
The $(\mathcal G,\mathcal R)$–consistency  
(Theorem~\ref{theo:consistency}) implies that\,%
\begin{equation}
    \sup_{j,i,x}
 \bigl|\hat r^{\,i}(x,j)-\psi(\eta^{\,i}_j(x))\bigr|\xrightarrow{p}0
\end{equation}

Because $\psi$ is strictly \emph{decreasing}, maximising
$\hat r^{\,i}(x,\cdot)$ is equivalent—w.p.\ $\to1$—to minimising
$\eta^{\,i}_j(x)$.  Write \(E_i(x)\) for this event.

On \(E(x)=E_s(x)\cap E_e(x)\) we have, for any $j_1,j_2$,
\[
   \hat S_{j_1}(x)>\hat S_{j_2}(x)
   \;\Longleftrightarrow\;
   R_{j_1}(x)<R_{j_2}(x),
\]
hence $\hat\pi(x)=\pi^{\ast}(x)$ on \(E(x)\).

By the union bound,
\(
   \Pr[\hat\pi(X)\neq\pi^{\ast}(X)]
      \le\Pr[\neg E_s(X)]+\Pr[\neg E_e(X)]
      \to0.
\)
\end{proof}

(i)  The link $\psi$ is unique for strictly proper
$\Phi_{01}^{\nu}$, hence no ambiguity arises.
(ii)  If surrogate optimisers are not uniquely defined, any
tie-breaking rule yields the same convergence result provided it is
deterministic.
(iii) The corollary formally justifies the inference
Algorithm~1 in the main paper: the single argmax over summed scores
is asymptotically optimal.

\subsection{Experiments}

\subsubsection{Few-Shot Demonstrations}\label{appendix:few_shot}
We present the few-shot demonstrations used to prompt the Llama-3 family of models. Datasets such as SQuADv2 contain questions for which no answer appears in the provided context. In these cases, we instruct the model to return no output, represented by the symbol `?`.

\begin{enumerate}
    \item \textbf{Demonstration 1:} 
    \begin{quote}
        Context:"The Eiffel Tower is located in Paris, France." \\
        Question: "Where is the Eiffel Tower?" \\
        Output: "Paris, France"
    \end{quote}
    \item \textbf{Demonstration 2:} 
    \begin{quote}
        Context: "Albert Einstein developed the theory of relativity in the early 20th century." \\
        Question: "What did Albert Einstein develop?" \\
        Output: "the theory of relativity"
    \end{quote}
    \item \textbf{Demonstration 3:} 
    \begin{quote}
        Context: "Marie Curie won the Nobel Prize in Physics in 1903 and in Chemistry in 1911." \\
        Question: "What year was Marie Curie born?"\\
        Output: "?"
    \end{quote}
    \item \textbf{Demonstration 4:}
    \begin{quote}
        Context: "The Great Wall of China was built to protect against invasions. It stretches over 13,000 miles."\\
        Question: "Who built the Great Wall of China?" \\
        Output: "?"
    \end{quote}
\end{enumerate}

\subsubsection{Agent Training and Performance Details}\label{appendix:agents_details}
We train our models on a single NVIDIA H100 GPU and report results averaged over four independent runs. We train both ALBERT-Base and ALBERT-XXL offline and will publicly release their weights. We do not train Llama-3.2-1B or Llama-3-8B from scratch; instead, we use the publicly available \emph{meta-llama} weights from HuggingFace without further training. For each dataset, we use the following hyperparameters:

\begin{table}[H]
\centering
\caption{Hyperparameters for SQuADv1, SQuADv2, and TriviaQA.}
\vspace{1\baselineskip}
\resizebox{0.8\textwidth}{!}{ 
\begin{tabular}{@{}cccccccc@{}}
\toprule
Experts & Batch Size & Epochs & Learning Rate & Warm-up & Scheduler & Max Length & Stride  \\
\toprule
ALBERT-Base & 32 & 2 & 5e-5 & 0.1 & linear & 384 & 128   \\
\midrule
ALBERT-XXL & 32 & 2 & 5e-5 & 0.1 & linear & 384 & 128  \\
\bottomrule
\end{tabular}}
\end{table}
\vspace{1\baselineskip}

We report the following performance metrics on a test set formed by subsampling 3,000 inputs from the validation set:

\begin{table}[H]
\centering
\caption{Exact Match (EM) and F1 scores for each dataset.}
\vspace{1\baselineskip}
\resizebox{0.5\textwidth}{!}{ 
\begin{tabular}{@{}cccc@{}}
\toprule
Agents & SQuADv1 & SQuADv2 & TriviaQA  \\
\toprule
ALBERT-Base & 84.20/90.63 & 77.10/79.54 & 86.63/90.86  \\
\midrule
ALBERT-XXL & 89.37/94.91 & 84.07/86.57 & 91.63/94.21  \\
\midrule
Llama-3.2-1B & 49.93/60.12 & 35.00/38.79 & 41.30/48.02\\
\midrule
Llama-3-8B & 67.80/80.22 & 59.47/66.47 & 48.47/56.66\\
\midrule
Ensemble & 84.60/90.80 & 81.06/84.19 & 88.84/91.78 \\
\bottomrule
\end{tabular}}
\label{tab:performance_agents}
\end{table}
\vspace{1\baselineskip}

\begin{table}[H]
\centering
\caption{Computational efficiency of different models. We compare the number of parameters (in millions) and computational cost (in GFLOPs) for processing a sequence of length $L=384$. The rejector is substantially lighter, with only $4.39$M parameters and $0.15$ GFLOPs, making it well suited to deployment in resource-constrained environments.}
\vspace{1\baselineskip}
\resizebox{0.7\textwidth}{!}{ 
\begin{tabular}{@{}ccccccc@{}}
\toprule
 & Llama-3.2-1B & ALBERT-Base & ALBERT-XXL &  Llama-3-8B & Rejector & Ensemble \\
\toprule
Parameters (M) & 1240 & 11.10 & 206	& 8030 & 4.39 & 1457.1 \\
\midrule
GFLOPs & 373.66	& 32.68	& 928.08 & 2,680.06	& 0.15  & 1,334.42\\
\bottomrule
\end{tabular}}
\label{tab:parameters}
\end{table}
\vspace{1\baselineskip}

\begin{figure}[H]
    \centering
    \begin{tikzpicture}[
    block/.style={
        rectangle, draw, rounded corners,
        align=center,
        minimum height=1cm,
        minimum width=2cm,
        font=\small
    },
    arrow/.style={-{Stealth[scale=1.2]}, thick}
]

\node[block, fill=yellow!20] (input) {Question\\and Context};

\node[block, fill=blue!30, right=1cm of input] (albert) {TinyBERT\\Feature Extractor};

\node[block, fill=blue!30, right=1cm of albert] (cls) {CLS extraction};

\node[block, fill=red!30, right=1cm of cls] (mlp) {Classification Head};

\draw[arrow] (albert) -- (cls);
\draw[arrow] (cls) -- (mlp);

\node[draw, dashed, inner sep=10pt, rounded corners, fit=(albert)(mlp)] (sr) {};
\node[above=5pt of sr.north] {Rejector};

\node[block, fill=green!20, right=1cm of mlp] (output) {Rejection Score};

\draw[arrow] (input) -- (sr);
\draw[arrow] (sr) -- (output);

\end{tikzpicture}
    \captionsetup{skip=2\baselineskip}
    \caption{Rejector architecture. The input is processed by TinyBERT \citep{devlin2018bert}, which serves as the feature extractor. The resulting CLS token is then used by the classification head to predict the allocation.}
    \label{fig:rejector_architecture}
\end{figure}

\subsubsection{Rejector Training}
Using the architecture illustrated in Figure~\ref{fig:rejector_architecture}, we train the rejector according to the surrogate deferral loss (SDL) defined in Lemma~\ref{lemma:surrogate}. Specifically, we adopt the standard multiclass log-softmax loss
\[
\Phi_{01}^{\nu=1}(r^i, x, j) = -\log\left(\frac{e^{r^i(x,j)}}{\sum_{j' \in \mathcal{A}} e^{r^i(x,j')}}\right),
\]
and optimize the rejector using the procedure detailed in Algorithm~\ref{algo_training1}. For all datasets---SQuADv1, SQuADv2, and TriviaQA---we employ the following hyperparameters:

\begin{table}[ht]
\centering
\caption{Key training hyperparameters for L2D rejector.}
\vspace{1\baselineskip}
\begin{tabular}{@{}ccccccc@{}}
\toprule
lr & epochs & batch size & dropout & optimizer & warmup ratio & weight decay \\
\midrule
$1\mathrm{e}{-5}$ & 6 & 128 & 0.1 & AdamW & 0.1 & 0.001 \\
\bottomrule
\end{tabular}
\label{tab:training_hyperparams}
\end{table}
\vspace{1\baselineskip}

For the query routers from \citep{ding2024hybridllmcostefficientqualityaware}, we use the following hyperparameters across all variants.
\begin{table}[ht]
\centering
\caption{Key training hyperparameters for $r_{\text{trans}}$, $r_{\text{prob}}$ and $r_{\text{deter}}$.}
\vspace{1\baselineskip}
\begin{tabular}{@{}ccccccc@{}}
\toprule
lr & epochs & batch size & dropout & optimizer & warmup ratio & weight decay \\
\midrule
$1\mathrm{e}{-2}$ & 10 & 32 & 0.1 & AdamW & 0.1 & 0.001 \\
\bottomrule
\end{tabular}
\label{tab:training_hyperparams_router}
\end{table}
\vspace{1\baselineskip}